\documentclass[twoside]{article}
\usepackage[dvips]{graphicx}
\usepackage{amsmath}
\usepackage{amssymb}
\newcommand{\argmin}{\mathop{\rm argmin}}
\let\V=\boldsymbol
\let\C=\mathcal
\newcommand{\argmax}{\mathop{\rm argmax}}
\newtheorem{remark}{Remark}
\newtheorem{example}{Example}
\newtheorem{definition}{Definition}
\newtheorem{proposition}{Proposition}
\newtheorem{proof}{Proof}
\newtheorem{lemma}{Lemma}
\newtheorem{theorem}{Theorem}

\DeclareMathOperator*{\I}{I}

\DeclareMathOperator*{\sign}{sgn}

\newcommand{\vtheta}{\V{\theta}}
\newcommand{\vphi}{\V{\phi}}
\newcommand{\vhtheta}{\V{\hat{\theta}}}

\newcommand{\nmean}{\frac{1}{n}\sum_{i=1}^n}
\newcommand{\nsum}{\sum_{i=1}^n}
\newcommand{\sumx}{\sum_{\V{x}\in\C{X}}}
\newcommand{\sumz}{\sum_{\V{x}\in\C{Z}}}

\newcommand{\bq}{\bar{q}}

\newcommand{\br}[1]{\left< #1 \right>}

\newcommand{\tra}{\tilde{r}_{\alpha,\V{\theta}}}
\newcommand{\trad}{\tilde{r}_{\alpha',\V{\theta}}}
\newcommand{\trae}{\tilde{r}_{\alpha,\vtheta,\eta}}
\newcommand{\traed}{\tilde{r}_{\alpha',\vtheta,\eta}}
\newcommand{\tp}{\tilde{p}}
\newcommand{\ra}{r_{\alpha,\V{\theta}}}
\newcommand{\rad}{r_{\alpha',\V{\theta}}}
\newcommand{\bqz}{\bar{q}_{\V{\theta}_0}}
\newcommand{\ptra}{\frac{\partial \tilde{r}_{\alpha,\V{\theta}}}{\partial \V{\theta}}}
\newcommand{\ptrad}{\frac{\partial \tilde{r}_{\alpha',\V{\theta}}}{\partial \V{\theta}}}

\newcommand{\petra}{\frac{\partial \tilde{r}_{\alpha,\V{\theta}}}{\partial \epsilon}}
\newcommand{\petrad}{\frac{\partial \tilde{r}_{\alpha',\V{\theta}}}{\partial \epsilon}}

\newcommand{\vmu}{\V{\mu}}

\usepackage{color}

\usepackage{natbib}
\usepackage{authblk}
\usepackage{lastpage}

\begin{document}
\title{
Parameter Estimation for Unnormalized Discrete Models via Empirically Localized 
Deformed Bregman Divergence}
\author[1]{Takashi Takenouchi} 
\affil[1]{National Graduate Institute for Policy Studies (GRIPS) and Riken AIP\\
t-takenouchi@grips.ac.jp        }


\maketitle

\begin{abstract}
Estimation of parameter of probabilistic models is an important task in the field of machine learning.
For models of discrete variables, calculation of the normalization constant of model is sometimes 
difficult and a lot of researches have been done to avoid the calculation of the normalization constant.
In this paper, we tackle with the difficulty by combining
a technique of empirical localization and a deformed Bregman divergence.
The technique of empirical localization makes it possible to drastically reduce computational cost of 
the calculation of the normalization constant, and 
in addition, appropriate choice of the deformation for the Bregman divergence
can invest the proposed estimator with
various kinds of favorable statistical properties, such as efficiency or robustness against
outlier noise. 
\end{abstract}

\begin{center}
 {\bf Keywords}
\end{center}
 Bregman divergence, Unnormalized model, Asymptotic efficiency

  \section{Introduction}  
Parameter estimation for discrete probabilistic models is a fundamental challenge in machine learning. 
Highly expressive undirected graphical models, 
such as Boltzmann machines
\cite{hinton1986learning,AckleyHS85,amari1992information}
 and restricted Boltzmann machines (RBMs)
 \cite{hinton2010practical,hinton2012better},
are widely used to represent complex dependencies over discrete states. 
These models are typically defined as
$q_{\vtheta}(\V{x}) = \frac{\bar{q}_{\vtheta}(\V{x})}{Z_{\vtheta}}$
where $\bar{q}_{\vtheta} (\V{x}) = \exp(\psi_{\vtheta} (\V{x}))$ is unnormalized model and
$Z_{\vtheta} = \sum_{\V{x} \in \mathcal{X}} \bq_{\vtheta} (\V{x})$
is normalization constant (partition function).
A major computational bottleneck in applying maximum likelihood estimation (MLE) 
to these models lies in the evaluation of the partition function
$Z_{\vtheta}$.
For high-dimensional discrete spaces $\mathcal{X} = \{+1, -1\}^d$, the exact calculation of $Z_{\vtheta}$ 
requires an exhaustive summation over $2^d$ states, 
making it computationally intractable.
To circumvent this bottleneck, several approximation paradigms have been proposed. 
One prominent strategy relies on sampling-based approximations. 
For instance, {Contrastive Divergence} (CD) \cite{hinton2002tpe} 
avoids the explicit calculation of $Z_{\vtheta}$ 
by using brief Markov Chain Monte Carlo (MCMC) trajectories starting from the empirical distribution. 
Although the CD has achieved notable empirical success, it lacks strong guarantees of statistical consistency 
because the sampling error does not necessarily vanish unless the Markov chain fully converges.
Noise Contrastive Estimation (NCE) trains the unnormalized model 
by considering a binary classification task between the empirical distribution
and a known reference distribution.
This clever formulation bypasses the computationally expensive normalization constant
and can construct a consistent estimator
\cite{gutmann2010noise}.
An alternative strategy focuses  on local statistics as density shapes.
For continuous variables, {Score Matching} \cite{hyvarinen2005estimation} 
matches derivative of the log-density (the score) of model with that of the empirical distribution, completely bypassing $Z_{\vtheta}$. 
Under continuous Langevin dynamics, 
the score matching can be viewed as a limiting case of contrastive divergence
\cite{hyvarinen2007connections}. 
In discrete spaces, the score matching has been extended to contrast probability of a state with its ``neighbors'' by coordinate-wise variable flipping 
\cite{hyvarinen2007some,gutmann2012bregman,dawid2012proper,
gutmann2012bregman}.
For a specific class of models on discrete space,
\cite{takenouchi2015is} proposed  an asymptotically consistent estimator
with the Itakura-Saito distance. 
{Minimum Probability Flow} (MPF) is a dynamics-based estimator and 
considers a continuous-time Markov chain 
to transition probability mass from the empirical distribution to the model distribution 
\cite{ICML2011Sohl-Dickstein_480}. 
The MPF does not require MCMC sampling or normalization constant calculations, and 
ensures statistical consistency by establishing ergodicity over the discrete states. 
Mathematically, discrete score matching, MPF, and their variants can be unified under the elegant framework of 
\emph{Minimum Stein Discrepancy} (MSD) \cite{barp2019minimum}, 
which derives the \emph{Fisher Divergence}
\cite{lyu2012interpretationgeneralizationscorematching} 
as a special case. 
In continuous settings, similar paradigms for bypassing the partition function have also inspired modern generative modeling, 
such as \emph{Flow Matching} \cite{lipmanflow},
which trains continuous-time flows by matching vector fields. 
However, extension of continuous flow matching or score-based generative frameworks for high-dimensional discrete 
domains is non-trivial due to the non-differentiable nature of discrete state transitions.

In this paper, we focus on a computationally feasible localization paradigm,
{empirical localization} \cite{JMLR:v18:15-596}.
Rather than considering transition dynamics over states,
we utilize e-mixture \cite{AmariNagaoka00} of the model 
and the empirical distribution  (or empirically localized models). 
A prior work ~\cite{JMLR:v18:15-596}
combined localized unnormalized discrete models 
with a homogeneous $\gamma$-divergence,
yielding asymptotically consistent and efficient estimator without calculating $Z_{\vtheta}$. 
However, the method is restricted to the homogeneous $\gamma$-divergences
and its statistical flexibility is limited.
To resolve the limitation, 
we propose a highly flexible parameter estimation framework,
integrating Deformed Bregman Divergence 
with the empirical localization. 
Furthermore, we extend our framework to a Bayesian-like MAP estimator
to incorporate prior domain knowledge into the computationally efficient estimation with the unnormalized model.

The main contributions of this paper are summarized as follows:
\begin{itemize}
    \item \textbf{A Generalized Unnormalized Estimator:} 
We establish a parameter estimation framework for discrete probabilistic models that completely avoids calculating $Z_{\vtheta}$ or performing expensive MCMC sampling.
    \item \textbf{Tunable Statistical Trade-offs:} We prove that our proposed estimator is Fisher consistent regardless of the choice of 
	  divergence and deformation. 
	  Furthermore, we show that adjusting the deformation functions allows us to achieve either asymptotic efficiency equivalent to MLE
	  or B-robustness (bounded gross-error sensitivity) against outlier noise.
    \item \textbf{Bayesian-like MAP Extension:} We generalize the estimator with the empirical localization to incorporate prior domain knowledge, 
	  providing a computationally feasible MAP-like estimator.
\end{itemize}

  Basic setups are described in Section 2.
  In Section 3, we proposed the estimator and
  investigate statistical properties such as efficiency or robustness.
  In Section 4, we refer a relationship between the proposed estimator and 
  some related works.
  In Section 5, we numerically verify performance of the proposed estimator and conclude
  in Section 6.

\section{Deformed Bregman divergence and Empirical localization}
In this section, we introduce definition of deformed Bregman divergence and the framework of empirical localization. 

 \subsection{Setup}

  Let $\C{X}$ be a discrete space.
Typical examples of the discrete space are
 $\{+1,-1\}^d$ or the set of natural numbers $\{1,2,\ldots\}$.
  The bracket $\br{f}$ for a real-valued function $f$ on $\C{X}$ denotes 
sum of $f(\V{x})$ over $\C{X}$
  {\it i.e.}, $\br{f}=\sum_{\V{x}\in\mathcal{X}}f(\V{x})$. 
 In this paper, we consider a problem of  parameter estimation
 of a probabilistic model $\bq_{\vtheta}(\V{x})$
 on $\C{X}$ that is written as
 \begin{equation}
  \bq_{\vtheta}(\V{x})=\frac{q_{\vtheta}(\V{x})}{Z_{\vtheta}} 
   \label{normalized.model}
 \end{equation} 
 where $\vtheta$ is an $m$-dimensional vector of parameters,
 $q_{\vtheta}(\V{x})$ is an unnormalized model and  
 $Z_{\vtheta}=\br{q_{\vtheta}}$ is normalization constant.
  The restriction $\br{ q_{\vtheta}}=\sumx q_{\vtheta}(\V{x})=1$ does not necessarily hold for unnormalized models and 
  typically calculation of the normalization constant $Z_{\vtheta}$ is computationally intractable.
  Throughout the paper, we assume, without loss of generality, that the unnormalized model $q_{\vtheta}(\V{x})$
  is written as
  \begin{align}
   q_{\vtheta}(\V{x})=\exp(\psi_{\vtheta}(\V{x})), 
   \label{pseudo.model}
  \end{align}
  where $\psi_{\vtheta}(\V{x})$ be an arbitrary function on $\C{X}$, parameterized by $\vtheta$
  and differentiable with respect to $\vtheta$.
  Note that
  by setting $\psi_{\vtheta}(\V{x})$ as $\psi_{\vtheta}(\V{x})-\log
   Z_{\vtheta}$,
   the normalized model \eqref{normalized.model} can be represented as
    \eqref{pseudo.model}.
    Several examples are given below.
    \begin{example}
     \label{example.bel}
     A function $\psi_{\theta}(x)=\theta x$ is associated with
     Bernoulli distribution on $\C{X}=\{+1,-1\}$.
    \end{example}
 \begin{example}
  \label{example.boltzmann}
  A function
  $\psi_{\vtheta,k}(\V{x})=
  (x_1,\ldots,x_d,x_1x_2,\ldots,x_{d-1}x_d,x_1x_2x_3,\ldots)\vtheta$
  in which monomials of degree up to $k$ appear, 
  corresponds to
  a $k$-th order Boltzmann machine
  \cite{hinton1986learning,sejnowski1986higher}.
  \end{example}
 \begin{example}
  Let $\V{x}_o\in \{+1,-1\}^{d_1}$ and $\V{x}_h\in \{+1,-1\}^{d_2}$ be an
  observed vector and hidden vector, respectively, and  
  $\V{x}=\left(\V{x}_o^T,\V{x}_h^T\right)\in 
  \{+1,-1\}^{d_1+d_2}$ be a concatenated vector, where
  $T$ indicates transpose.
  The Boltzmann machine with hidden variables is defined as
  $q_{h,\vtheta}(\V{x}_o)=\exp(\psi_{h,\vtheta}(\V{x}_o))$ where
  the function $\psi_{h,\vtheta}(\V{x}_o)$ is 
    $\psi_{h,\vtheta}(\V{x}_o)=
    \log \sum_{\V{x}_h}\exp(\psi_{\vtheta,2}(\V{x}))$
  and $\sum_{\V{x}_h}$ denotes summation with respect to the hidden variable $\V{x}_h$. 
 \end{example}

   \begin{example}\label{ex.poisson}
    For Poisson distribution $p(x)=\frac{\theta^x e^{-\theta}}{x!}
    =\exp\left(x\log \theta -\log x! -\theta\right)$
    on $\{0,1,2,\ldots\}$,
    a function $\psi_{\vtheta}(x)$ is written as $x\log \theta-\log x!$.
    \end{example}

\begin{example}\label{ex.gaussian}
 Multivariate Lattice Gaussian Distribution on $\mathbb{Z}^p$ is defined by
\begin{align*}
p(\V{x}) \propto \exp\left(-\frac{1}{2}(\V{x}-\V{\mu})^T \Sigma^{-1} (\V{x}-\V{\mu})\right)
\end{align*}
where $\mathbb{Z}^p$ is the integer lattice, $\V{\mu}$ is the mean vector, and
$\Sigma$ is a covariance matrix.
\end{example}

     Let $\C{D}=\{\V{x}_i\}_{i=1}^n$ be a dataset generated from an underlying distribution $p(\V{x})$.
     and $\C{Z}$ be a set of unique patterns in the dataset $\C{D}$. 
     An empirical distribution $\tilde{p}(\V{x})$ associated with the
  dataset $\C{D}$ is defined as
  \begin{equation}
   \tilde{p}(\V{x})=
    \begin{cases}
     \frac{n_{\V{x}}}{n} & \V{x}\in \C{Z},\\
     0 & \mbox{otherwise},
    \end{cases}
    \label{empirical.distribution}
  \end{equation}
  where $n_{\V{x}}$ is the empirical count of pattern $\V{x}$ in the dataset
  $\C{D}$.
  MLE is a typical choice to estimate the parameter $\vtheta$, 
      \begin{align*}
       \vhtheta_{\mathrm{mle}}=\argmax_{\vtheta}L(\vtheta),
      \end{align*}
     where
     $
      L(\vtheta)=\nmean \log \bq_{\vtheta}(\V{x}_i)
      =\br{\tilde{p}\log\bq_{\vtheta}}
      =\br{\tilde{p}\psi_{\vtheta}}-
      \br{\log Z_{\vtheta}}
      $
     is the log-likelihood of $\vtheta$ with the normalized model $\bq_{\vtheta}$. 
     MLE is asymptotically consistent and efficient but does not
     have an explicit solution in general.
     Then 
     the maximization of the log-likelihood function can be computationally demanding for normalized
     models on huge discrete sample spaces.
     In fact,
     the gradient of $L(\vtheta)$ is written as
     $\frac{\partial L}{\partial \vtheta}=\br{\tilde{p}\psi_{\vtheta}'}-\br{\bq_{\vtheta}\psi_{\vtheta}'}$
     where $\psi_{\vtheta}'=\frac{\partial \psi_{\vtheta}}{\partial
     \vtheta}$ and
     the second term $\br{\bq_{\vtheta}\psi_{\vtheta}'}$ is the expectation of $\psi_{\vtheta}'$
     with respect to the current model $\bq_{\vtheta}$.
     This calculation requires large computational cost because of the calculation of the normalization constant, while
     the first term $\br{\tilde{p}\psi_{\vtheta}'}$ is just an empirical mean and is easily calculated.

     \subsection{Deformed Bregman divergence}
     As a preparation for a proposed  estimator without the calculation of the normalization constant,
     we define a deformed Bregman divergence
     for two arbitrary positive measures $p,q$ on $\C{X}$:
     \begin{definition}
      Let $U$ be a strictly convex generating function and $f$ be a monotonically increasing
   function. Then the deformed Bregman divergence for $p,q$
   is defined as
  \begin{align}
    D(p,q;U,f)&=
   \br{\Phi_{U,f}(p,q)}
      \label{deformed.bregman}
  \end{align}
   where 
   \begin{align}
   \Phi_{U,f}(p,q)=U(f(q))-U(f(p))-U'(f(p))\left(f(q)-f(p)\right).   
    \nonumber
    \end{align}
     \end{definition}
  Note that the function $\Phi_{U,f}$ is always non-negative because of the
  convexity of $U$.

   \begin{example}
    The conventional Bregman divergence is derived with a function $f(z)=z$.
   \end{example}
  
\begin{example}
  If we employ $f=(U')^{-1}$ (an inverse function of $U'$), \eqref{deformed.bregman} reduces
  to the statistical version of Bregman divergence~\cite{Murata_etal02},
  \begin{align*}
   D(\tilde{p},q;U,f=(U')^{-1})
   &=
    Const+\br{U(f(q))-\tilde{p}f(q)}
   \\
   &=
   Const+\br{U(f(q))}-\nmean f(q(\V{x}_i)).
   \end{align*}
\end{example}
For this divergence, we can directly plugged-in the empirical distribution $\tilde{p}(\V{x})$,
even on a continuous space.
 
 Note that estimators constructed by minimizing
 \begin{align*}\\
  D(\tilde{p},q_{\vtheta};U,f)  &=
  Const+\br{U(f(q_{\vtheta}))-U'(f(\tilde{p}))f(q_{\vtheta})}
  \end{align*}
   or
\begin{align*}
D(\tilde{p},\bq_{\vtheta};U,f) &=
 Const+\br{U(f(\bq_{\vtheta}))-U'(f(\tilde{p}))f(\bq_{\vtheta}) }
\end{align*}
 do not resolve the problem of computational cost associated with the
 normalization constant $Z_{\vtheta}$, because
 calculation of $\br{U(f(q_{\vtheta}))}$ or $\br{U(f(\bq_{\vtheta}))}$
 requires
 the same order of computational cost with that of $Z_{\vtheta}$.

 \subsection{Empirical Localization}
 The empirical localization of the
    unnormalized model \eqref{pseudo.model} by $p(\V{x})$ (or $\tilde{p}(\V{x})$) with ratio
     $\alpha$ is defined as follows.
   \begin{definition}
    For $\alpha\in R$,
    the empirically localized model $r_{\alpha,\vtheta}(\V{x})$
    ($\tilde{r}_{\alpha,\vtheta}(\V{x})$) of
    $\bq_{\vtheta}(\V{x})$
    and distributions
    $p(\V{x})$ ($\tilde{p}(\V{x})$) 
    ~\cite{takenouchi2015empirical} is defined by
   \begin{align*}
    r_{\alpha,\vtheta}(\V{x})&=\frac{p(\V{x})^{\alpha}\bq_{\vtheta}(\V{x})^{1-\alpha}}
     {\br{p^{\alpha}\bq_{\vtheta}^{1-\alpha}}}
     =\frac{p(\V{x})^{\alpha}q_{\vtheta}(\V{x})^{1-\alpha}}
     {\br{p^{\alpha}q_{\vtheta}^{1-\alpha}}},
    \\
    \tilde{r}_{\alpha,\vtheta}(\V{x})&=\frac{\tilde{p}(\V{x})^{\alpha}\bq_{\vtheta}(\V{x})^{1-\alpha}}
     {\br{\tilde{p}^{\alpha}\bq_{\vtheta}^{1-\alpha}}}
     =
     \frac{\tilde{p}(\V{x})^{\alpha}q_{\vtheta}(\V{x})^{1-\alpha}}
    {\br{\tilde{p}^{\alpha}q_{\vtheta}^{1-\alpha}}}.
    \end{align*}

   \end{definition}
      Note that
      $r_{0,\vtheta}(\V{x})=\tilde{r}_{0,\vtheta}(\V{x})=\bq_{\vtheta}(\V{x})$,
      $r_{1,\vtheta}(\V{x})=p(\V{x})$, and
      $\tilde{r}_{1,\vtheta}(\V{x})=\tilde{p}(\V{x})$
      hold.
      The empirically localized model is a probability distribution, {\it i.e.},
      $\br{r_{\alpha,\vtheta}}=\br{\tra}=1$ holds and 
      can be constructed only with the
      unnormalized model \eqref{pseudo.model} because 
      the normalization constant $Z_{\vtheta}$ of
      $\bq_{\vtheta}(\V{x})$ is canceled out.      
      Also if $p(\V{x})=\bq_{\vtheta_0}(\V{x})$,
      $r_{\alpha,\vtheta_0}(\V{x})=\bq_{\vtheta_0}(\V{x})$ holds for an arbitrary $\alpha$.

      The empirically localized model
      $\tilde{r}_{\alpha,\vtheta}$ is rewritten as
      \begin{equation*}
       \tilde{r}_{\alpha,\vtheta}(\V{x})
	=
	\begin{cases}
	 \frac{n_{\V{x}}^{\alpha}\exp\left((1-\alpha)\psi_{\vtheta}(\V{x})\right)}
	 {\sumz n_{\V{x}}^{\alpha}\exp\left((1-\alpha)\psi_{\vtheta}(\V{x})\right)}
	 & \V{x}\in \C{Z},\\
	 0 & \mbox{otherwise}.
	\end{cases}
      \end{equation*}
      Note that 
      the term $\tilde{r}_{\alpha,\vtheta}$ can be calculated
      only on the domain $\C{Z}$ of observed examples and we can ignore
      a domain of unobserved examples.
      This implies that the calculation of $\tilde{r}_{\alpha,\vtheta}$
      requires $\C{O}(n)$ summations and
      computational cost for $\tilde{r}_{\alpha,\vtheta}$
      is drastically reduced compared with those of
      $\bq_{\vtheta}$ or $Z_{\vtheta}$
      (which can be $\C{O}(2^d)$ for $\C{X}=\{\pm 1\}^d$).

  \section{Proposed Estimator and statistical properties}
  \label{proposed} 
  In this section, 
  we propose an estimator with 
  the empirical localization and 
  the deformed Bregman divergence,
  which can omit calculation of $Z_{\vtheta}$ and can drastically reduces computational cost.
  Also we investigate statistical aspects of the proposed estimator. 
  The proposed estimator is defined by
  minimization of the deformed Bregman divergence between two different empirically localized models, as
  \begin{align}
   \vhtheta_{U,f} &=
   \argmin_{\vtheta}D(\tra,\trad;U,f),
   \label{estimator}
  \end{align}
  where $\alpha,\alpha'\not=0$, $\alpha\not=\alpha'$.
  Note that the estimator \eqref{estimator} depends on choice of 
  $\alpha,\alpha'$, the function $U$, and $f$, but we omit $\alpha,\alpha'$ for simplicity of notations.  
 \begin{remark}
  The proposed estimator does not require the calculation of
  the normalization constant $Z_{\vtheta}$ because
  the normalization constant is canceled out in
  the empirically localized model $\tilde{r}_{\alpha,\vtheta}$.
  Also, its computational cost is $\C{O}(n)$ because we can ignore the
  domain of unobserved examples.
 \end{remark}
 The minimization problem \eqref{estimator} can be solved
 using arbitrary optimization methods, and we do not care about the point in
 this paper. 

\subsection{Information geometry of the proposed estimator}
     An information geometrical interpretation of the proposed estimator is shown in
    Figure \ref{fig.estimator}.
    In the scheme of the information geometry
    ~\cite{AmariNagaoka00},
    we regard a probability
    distribution as a point on a manifold and discuss geometrical structures of the
    statistical manifold consist of probability distributions.
    From the viewpoint of the information geometry,
    MLE or minimization of KL divergence
    can be interpreted as orthogonal projection of
    the empirical distribution $\tilde{p}$ onto
    the statistical manifold $\bq_{\vtheta}$.
    On the other hand,
    the proposed estimator first
    considers two empirically localized models $\tra$ and $\trad$, which lie on points 
    between $\tilde{p}$ and the statistical manifold $\bq_{\vtheta}$.
    The curve connecting $\tilde{p}$ and $\bq_{\vtheta}$ is called ``e-geodesic''
    and $\tra$ and $\trad$
    are (extended) dividing points associated with
    ratios $\alpha, \alpha'$.
    Then we consider minimization of the deformed Bregman divergence between
    $\tra$ and $\trad$ as shown in Figure \ref{fig.estimator}.
    In the following subsection, we theoretically investigate 
    how the proposed estimator behaves.
   \begin{figure}[ht]
\begin{center}
 \centerline{\includegraphics[width=.85\columnwidth,bb=0 0 842 595]{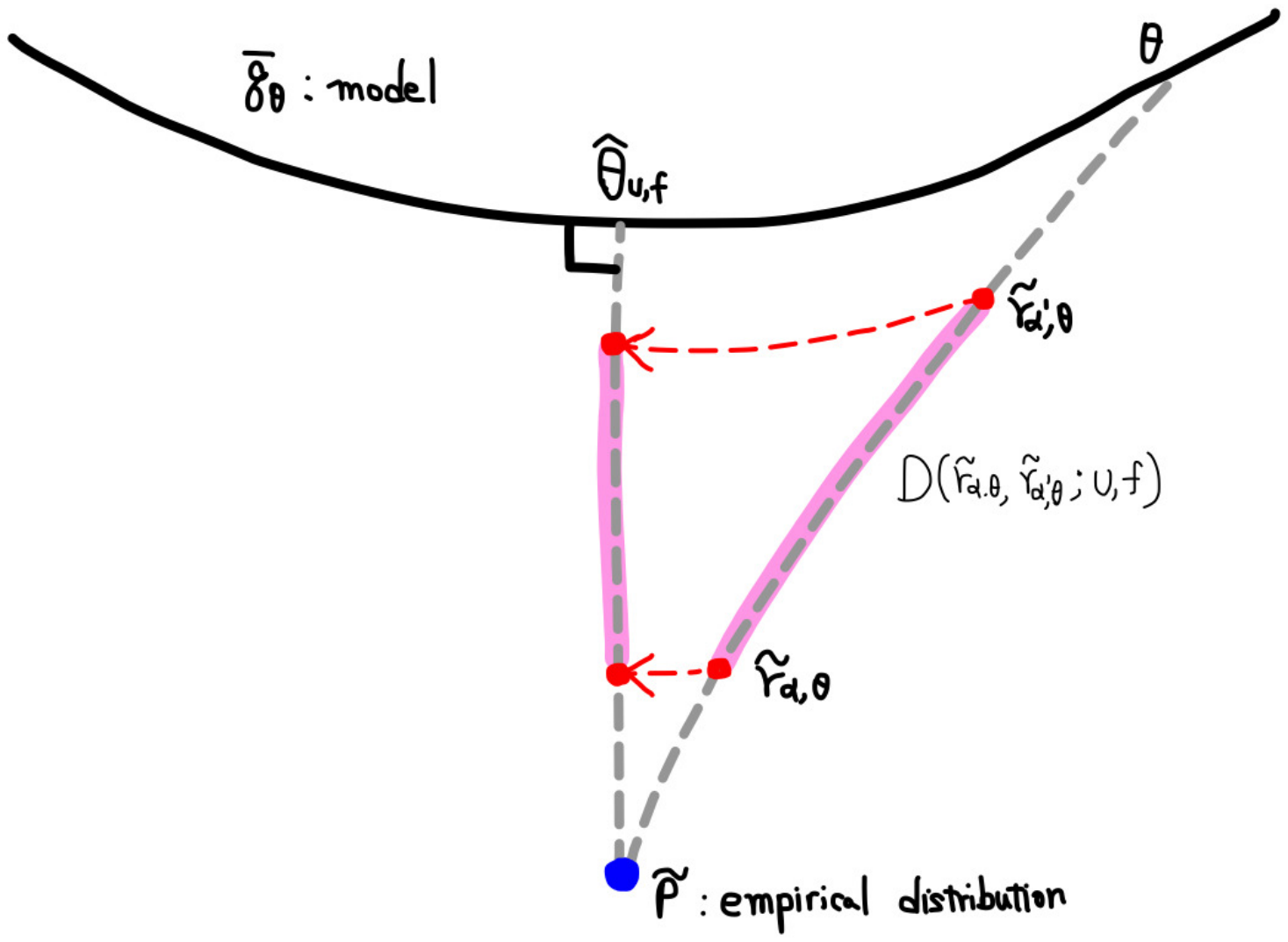}}
 \caption{An information geometrical interpretation of the proposed estimator.
 Each point in the figure corresponds to a probability distribution.
 Note that the deformed Bregman divergence between two points on a gray curve
 is computationally tractable.
       }
       \label{fig.estimator}
\end{center}
\end{figure}
  \subsection{Statistical Properties of Proposed estimator}
    While detailed statistical properties of the proposed estimator may 
    depend on a choice of
    $U$ and $f$, the estimator \eqref{estimator} is Fisher consistent
    regardless of choice of $U$ and $f$, as follows.
    \begin{proposition}
     When the underlying distribution of dataset is written as $p(\V{x})=\bq_{\vtheta_0}(\V{x})$,
    $\vtheta_0$ is a stationary point of the risk function $D(\ra,\rad;U,f)$
    and the estimator $\vhtheta_{U,f}$ is Fisher consistent, {\it i.e.},
    \begin{equation}
     \vtheta_0=\argmin_{\vtheta}D(\ra,\rad;U,f)
     \end{equation}
    holds.
    \end{proposition}
   \begin{proof}
   The equilibrium condition of the estimator associated
   $p=\bq_{\vtheta_0}$
    is written as
    \begin{align}
     &0=
    \left.
    \frac{\partial }{\partial \vtheta} D(\ra,\rad;U,f)
    \right|_{\vtheta=\vhtheta_{U,f}}
    \nonumber \\
    =&
    \left<
    (1-\alpha')\left\{U'(f(\rad))-U'(f(\ra))\right\}f'(\rad)
    \rad (\psi'_{\vtheta}-{\vmu}_{\alpha'})
    \right. \nonumber \\
     - &   
    \left.
    \left.
    (1-\alpha)(f(\rad)-f(\ra))U''(f(\ra))f'(\ra)\ra
    (\psi'_{\vtheta}-{\vmu}_{\alpha})
    \right>
    \right|_{\vtheta=\vhtheta_{U,f}},
    \label{equiv.condition}
    \end{align}
   where ${\vmu}_{\alpha}=\br{ \ra \psi'_{\vtheta}}$.
   Because 
   ${r}_{\alpha,\vtheta_0}=\bq_{\vtheta_0}$
   holds
   for an arbitrary $\alpha$,
   $\vhtheta_{U,f}=\vtheta_0$ satisfies
   \eqref{equiv.condition}.
   \end{proof}
Furthermore, we obtain the following asymptotic property of the estimator.
  \begin{lemma}\label{lemma.1}
   The asymptotic distribution of $\vhtheta_{U,f}$
   is given as multivariate Gaussian distribution,   
   \begin{equation}
    \sqrt{n}(\vhtheta_{U,f}-\vtheta_0) \sim N(0,V_{U,f,\vtheta_0})
   \end{equation}
   where $V_{U,f,\vtheta_0}$ is an asymptotic variance,
   \begin{equation}
    V_{U,f,\vtheta_0}=H_{U,f,\vtheta_0}^{-1}J_{U,f,\vtheta_0}H_{U,f,\vtheta_0}^{-1}.
     \label{variance.uf}
   \end{equation}
   $H_{U,f,\vtheta}$ and $J_{U,f,\vtheta}$ are matrices defined as
      \begin{align}
    H_{U,f,\vtheta_0}&=\br{
     \xi_{U,f}(\bq_{\vtheta_0})
     \bq_{\vtheta_0}
     (\psi'_{\vtheta_0}-\vmu_0)
       (\psi'_{\vtheta_0}-\vmu_0)^T},
      \nonumber \\
       J_{U,f,\vtheta_0}&=\br{\bq_{\vtheta_0}\zeta_{U,f,\vtheta_0}\zeta_{U,f,\vtheta_0}^T},
\nonumber
      \end{align}
   where   $\vmu_0=\br{\bq_{\vtheta_0}\psi_{\vtheta_0}'}$,
   \begin{align}
   \zeta_{U,f,\vtheta}(\V{x})&=
    \xi_{U,f}(\bq_{\vtheta_0}(\V{x}))
      (\psi_{\vtheta_0}'(\V{x})-\vmu_0      )
      -
      \br{\bq_{\vtheta_0}\xi_{U,f}(\bq_{\vtheta_0})
      (\psi_{\vtheta_0}'-\vmu_0)},\nonumber \\
    \xi_{U,f}(z)&=U''(f(z))f'(z)^2z.
     \label{xi}
   \end{align}
  \end{lemma}
     \begin{proof}
      See Appendix A.
     \end{proof}
    While Figure \ref{fig.estimator} intuitively shows that 
    behavior of 
    the estimator with $\alpha\simeq 1$ and $\alpha'\simeq 0$
    may be similar to that of MLE,
    the asymptotic analysis shows that the efficiency of estimator
    does not depend on choice of $\alpha,\alpha'$.
    Note that the efficiency of the estimator depends on a choice of 
    the generating function $U$ of Bregman divergence and
    the deformation function $f$.
   \begin{theorem}
    \label{theorem.efficiency}
   When the function $f$ satisfies
   \begin{equation}
    \xi_{U,f}(z)=U''(f(z))f'(z)^2z=1,
     \label{condition.f}
   \end{equation}
   the estimator $\vtheta_{U,f}$ is asymptotically efficient, {\it i.e.},
   the asymptotic variance \eqref{variance.uf} is equal to   inverse of 
   Fisher information matrix
    $\br{\bq_{\vtheta_0}} (\psi'_{\vtheta_0}-\vmu_0)
       (\psi'_{\vtheta_0}-\vmu_0)^T$.   
   \end{theorem}
  \begin{proof}
   When the function $f$ satisfies \eqref{condition.f},
   we observe that
   $H_{U,f,\vtheta_0}$ and $J_{U,f,\vtheta_0}$
   are equal to Fisher information matrix
   $\br{   
   \bq_{\vtheta_0}
   \left(\psi_{\vtheta_0}'-\vmu_0\right)
   \left(\psi_{\vtheta_0}'-\vmu_0\right)^T}$,
   which concludes the theorem.
  \end{proof}
  
    Note that the proposed estimator does not require calculation of the
    normalization constant $Z_{\vtheta}$
    and its computational cost can be drastically reduced compared to
    conventional estimator such as  MLE, 
    while the proposed estimator attains asymptotic efficiency
    if we choose the appropriate deformation function $f$ 
    for a generating function $U$.
    
  \begin{example}
   \label{ex.kl}
   For $U(z)=\exp(z)$, $f(z)=\log z$ satisfies the condition \eqref{condition.f}
   and the deformed Bregman divergence reduces to the extended
   KL-divergence, and associated risk function is written as
   $\br{\tra\log\frac{\tra}{\trad}}$.
  \end{example}
   \begin{example}\label{ex.beta}
   For $U(z)=\frac{1}{1+\beta}(\beta z+1)^{(\beta +1)/\beta}$ associated
   with the $\beta$-divergence
   \cite{Murata_etal02},
    $f(z)=\frac{(1+\beta)^{2\beta/(1+\beta)}}{\beta}
     z^{\frac{\beta}{1+\beta}}
     -\frac{1}{\beta}$
   satisfies the condition \eqref{condition.f} and the deformed $\beta$-divergence is
   written as
   \begin{equation}
    D(p,q;U,f)=(1+\beta)\br{q+\frac{p}{\beta}-\frac{1+\beta}{\beta}p^{\frac{1}{1+\beta}}q^{\frac{\beta}{1+\beta}}}
   \end{equation}
   which is equivalent to the $\alpha$-divergence
   ~\cite{amari2009alpha}.
    Because $\br{\tra}=\br{\trad}=1$ holds, the estimator can be calculated by
    minimizing
    \begin{align*}
     -\br{\tra^{\frac{1}{1+\beta}}\trad^{\frac{\beta}{1+\beta}}}.
    \end{align*}
   \end{example}
   \begin{example}
    If we set $\beta=1$ in the above Example \ref{ex.beta},
    the generating function of the Bregman divergence is $U(z)=\frac{z^2}{2}$ and $f(z)=2\sqrt{z}$ satisfies
   \eqref{condition.f}
   and the associated deformed Bregman divergence is Hellinger distance,
   \begin{equation}
    D(p,q;U,f)=2\br{(\sqrt{q}-\sqrt{p})^2}
   \end{equation}
    and the risk function of the proposed estimator is 
    $D(\tra,\trad;U,f)=4-4\br{\sqrt{\tra\trad}}$.
   \end{example}
   \begin{example}\label{example.is}
   For $U(z)=-\log (-z)$ which is associated with the Itakura-Saito
   distance, $f_{\pm}(z)=-\exp(\pm 2\sqrt{z})$ satisfies \eqref{condition.f}, and
   the deformed Itakura-Saito distance is written as
   \begin{equation}
    D(p,q;U,f_{\pm})=\br{\mp 2(\sqrt{p}-\sqrt{q})+e^{\pm 2(\sqrt{p}-\sqrt{q})}-1}.
   \end{equation}
    Note that $D(p,q;U,f_+)=D(q,p;U,f_-)$ holds.
   \end{example}

  \subsection{Robust estimation with the deformed Bregman divergence}
  \label{sec.robust}
  While an appropriately selected deformation function $f$ can attain
  the asymptotic efficiency
  as shown in the previous subsection,
  the proposed estimator can acquire various kinds of properties by
  selecting functions $U$ and $f$ at the cost of the asymptotic efficiency.
  In this section, we focus on robustness of the proposed estimator
  against outlier noise.

  At first, we introduce a typical measure of robustness, the influence function
  \cite{hampel2011robust}.
   \begin{definition}
    Let  $p(\V{x})=\bq_{\vtheta_0}(\V{x})$ be the underlying
    distribution and 
    $p_{\varepsilon,\tilde{\V{x}}}(\V{x})=(1-\varepsilon)p(\V{x})+\varepsilon \I(\V{x}=\tilde{\V{x}})$ be a distribution
    contaminated by an outlier noise $\tilde{\V{x}}$,
   where $\varepsilon$ is a constant in $[0,1]$ and
    $\I$ is an indicator function. 
    Let $\vhtheta_{\varepsilon}$ be an estimator
    constructed with the contaminated distribution $p_{\varepsilon,\tilde{\V{x}}}(\V{x})$.
    Then an influence function of the estimator
    is defined as
      \begin{align}
    {\rm IF}(\tilde{\V{x}})&=\lim_{\epsilon\to 0}
    \frac{\vhtheta_{\epsilon}-\vtheta_0}{\epsilon}.
      \end{align}
   \end{definition}
      The influence function is a quantity representing 
   how the estimator is influenced
   by a
   small proportion of outlier noise $\tilde{\V{x}}$.
   \begin{definition}
    The gross-error sensitivity is defined by
    $\sup_{\tilde{\V{x}}}| {\rm IF}(\tilde{\V{x}}) |$, which is a basic
    characteristic of robustness, and
    an estimator having finite gross-error
    sensitivity is said to be B-robust.
   \end{definition}
   
   \begin{lemma}
    \label{lemma.if}
    Let $\vhtheta_{U,f,\varepsilon}$ be the proposed estimator \eqref{estimator}
    associated with 
    $p_{\varepsilon,\tilde{\V{x}}}(\V{x})$.
    Then an influence function of the estimator $\vhtheta_{U,f,\varepsilon}$
    is written as
      \begin{align}
       {\rm IF}(\tilde{\V{x}})
       \propto &
       \xi_{U,f}(\bq_{\vtheta_0}(\tilde{\V{x}}))
       \left\{
    \psi'_{\vtheta_0}(\tilde{\V{x}})
    -\vmu_0
       \right\}
       -\br{
    \bq_{\vtheta_0}\xi_{U,f}(\bq_{\vtheta_0})
    \left\{
    \psi'_{\vtheta_0}
       -\vmu_0
       \right\}
       }
       \label{influence.function}
      \end{align}
    where the function $\xi_{U,f}$ is defined in \eqref{xi}.
   \end{lemma}
   \begin{proof}    
    By expanding 
    the equilibrium condition of the estimator
    $\vhtheta_{U,f,\varepsilon}$
    associated the
    contaminated distribution around $\varepsilon=0$,
    we obtain 
    \begin{align}
     0\simeq &
      \left.
      \frac{\partial }{\partial \vtheta}
      D(\tra,\trad;U,f)
      \right|_{\vtheta=\vtheta_0}+
      \left.\frac{\partial^2}{\partial \vtheta\partial \vtheta^T}
      D(\tra,\trad;U,f)\right|_{\vtheta=\vtheta_0}
      (\vhtheta_{U,f,\epsilon}-\vtheta_0)
     \nonumber \\ 
     =&
-\epsilon (\alpha-\alpha')^2
      \br{\xi_{U,f}(\bq_{\vtheta_0})(\psi_{\vtheta_0}'-\vmu_0)
      (\I(\V{x}=\tilde{\V{x}})-q)
      }
     \nonumber \\
     &
+
      \left.\frac{\partial^2}{\partial \vtheta\partial \vtheta^T}
      D(\tra,\trad;U,f)\right|_{\vtheta=\vtheta_0}
      (\vhtheta_{U,f,\epsilon}-\vtheta_0).
     \label{lemma.if2}
    \end{align}
    The detailed calculation is exactly the same as 
    the derivation of \eqref{eq.d2} shown in Appendix A,
    and can be obtained immediately by replacing $\tp(x)=\bq_{\vtheta_0}(\V{x})+\epsilon s(\V{x})$ 
with $p_{\epsilon,\tilde{\V{x}}}=(1-\epsilon)p(\V{x})+\epsilon \I(\V{x}=\tilde{\V{x}})=
p(\V{x})+\epsilon \left(\I(\V{x}=\tilde{\V{x}})-p(\V{x})\right)$.
    By dividing \eqref{lemma.if2} with $\epsilon$, we obtain \eqref{lemma.if}.
   \end{proof}
   We observe that
   the first term of \eqref{influence.function} is dominant 
   for the influence function.
   \begin{remark}
       If $\xi_{U,f}(z)=1$ and $\psi'_{\vtheta}$ is 
     bounded,
    the estimator $\vhtheta_{U,f}$ is asymptotically efficient, and
    the gross-error sensitivity $\sup_{\tilde{\V{x}}}|{\rm
    IF}(\tilde{\V{x}})|$ is bounded, {\it i.e.}, the estimator is B-robust.
    For example, 
    the function $\psi_{\vtheta}'(\V{x})$ associated with Boltzmann machine on 
    $\C{X}=\{+1,-1\}^p$ is always bounded .
   \end{remark}
\begin{remark}
 When we employ the
 Poisson distribution (Example \ref{ex.poisson}) or
 the Multivariate Lattice Gaussian Distribution (Example \ref{ex.gaussian}),
$\psi_{\vtheta}'(\V{x})$ is not bounded and then 
 the influence function can diverge.
\end{remark}

      \begin{theorem}
       \label{theorem.robust}
       Assume that
       $\psi_{\vtheta}(\V{x})=\vtheta^T\vphi(\V{x})$ and
       \begin{equation}
	\xi_{U,f}(z)=z(1-z)
	 \label{condition.robust}
       \end{equation}
       holds.      
      Under the same conditions in the Lemma \ref{lemma.if},
         \begin{align*}
    \lim_{|\psi (\tilde{\V{x}})|\to \infty}\bq_{\vtheta_0}(\tilde{\V{x}})
    (1-\bq_{\vtheta_0}(\tilde{\V{x}}))
    \vphi(\tilde{\V{x}})=0
   \end{align*}
      holds and
      $\sup_{\tilde{\V{x}}}|{\rm IF}(\tilde{\V{x}})|$ is bounded
       even when
       $\psi( \tilde{\V{x}})$ diverges.
      \end{theorem}

      The theorem indicates that
      the estimator \eqref{estimator} is B-robust and is not influenced so much by the outlier
      when the function $f$ satisfies the condition
      \eqref{condition.robust}.
      The influence caused by  the outlier can be also problematic under
      the regression or classification tasks.
      For example, if we consider the logistic regression model as
      \begin{equation}
       \bq_{\vtheta}(y|\V{x})=\frac{\exp(\psi_{\vtheta}(\V{x},y))}{\sum_{y'}\exp(\psi_{\vtheta}(\V{x},y'))},
      \end{equation}
      the function $\psi_{\vtheta}(\V{x},y)$ can diverge because of the
      outlier noise $\tilde{\V{x}}$, and in such situation, the property of robustness shown in Theorem
      \ref{theorem.robust}
      becomes effective.

     We find some examples satisfying 
     \eqref{condition.robust}
     in Theorem \ref{theorem.robust}.
   
     \begin{example}
      \label{example.robust1}
     If we employ
     $U(z)=z\log z-z$,    
     $f(z)=\frac{1}{9}|z-1|^3$ satisfies the condition \eqref{condition.robust}
     and the associated divergence is written as
   \begin{align*}
    D(p,q;U,f)
    =&\frac{1}{9}\br{
     |q-1|^3\log\frac{|q-1|^3}{|p-1|^3}-
     |q-1|^3+|p-1|^3
     }
    \\
    \propto &
    KL(|q-1|^3,|p-1|^3).
    \end{align*}
     \end{example}

    \begin{example}
     If we employ
     $U(z)=\frac{1}{2}z^2$,
     $f(z)=\frac{2}{3}\sign (z-1)|z-1|^{3/2}$
     satisfies the condition \eqref{condition.robust}
     and the associated divergence is written as
   \begin{align*}
    D(p,q;U,f)
    =
    \frac{2}{9}
    \br{
    \left\{
    \sign(q-1)|q-1|^{\frac{3}{2}}
    -
    \sign(p-1)|p-1|^{\frac{3}{2}}
    \right\}^2
    }.
   \end{align*}
    \end{example}

\section{Related works}

  \subsection{Relation with \cite{JMLR:v18:15-596}}\label{takekana}
  While the estimator \eqref{estimator} is defined 
  using the deformed Bregman divergence,
  a similar estimator has been proposed using the $\gamma$-divergence 
  to avoid intractable calculation of the normalization constant
  \cite{JMLR:v18:15-596}.
  The $\gamma$-divergence is defined as
\begin{align*}
 D_{\gamma}(p,q)=\frac{1}{1+\gamma}\log \br{p^{1+\gamma}}
+\frac{\gamma}{1+\gamma}\log \br{q^{1+\gamma}}
-\log \br{pq^\gamma}
\end{align*}
where $\gamma >0$.
The divergence satisfies $D_{\gamma}(p,q)\geq 0$, and $D_{\gamma}(p,q)=0$
if and only if $p\propto q$ rather than $p=q$ \cite{fujisawa08:_robus}.
In the limit of $\gamma\to 0$, the divergence converges to the extended KL-divergence.
An estimator constructed by 
applying the empirically localized model for the $\gamma$-divergence
can also bypass the calculation of the normalization constant and 
an asymptotic variance of the estimator coincides with that MLE, {\it i.e.},
the estimator is asymptotically efficient.

Though both estimators can attain the asymptotic efficiency, 
an advantage of the proposed estimator is
flexibility obtained by the deformation such as the efficiency or robustness.

     \subsection{Relation with contrastive divergence}
  In this subsection,
     we discuss a relationship between the proposed method and 
     contrastive divergence.
     Contrastive divergence approximates the calculation of the
     normalization constant $Z_{\vtheta}$ using the technique of 
     MCMC,
     in which a markov chain runs few times from the empirical
     distribution to the targeted statistical model $\bq_{\vtheta}$, as
     \begin{equation*}
      \frac{\partial L(\vtheta)}{\partial \vtheta}
       =
       \br{\tilde{p}\psi'_{\vtheta}}-\br{\bq_{\vtheta}\psi'_{\vtheta}}
       \simeq
       \br{\tilde{p}\psi'_{\vtheta}}-\br{\bar{p}'\psi'_{\vtheta}}
     \end{equation*}
     where $\bar{p}$  is a distribution made by the markov chain.
     On the other hand, the proposed estimator utilizes the empirically localized model rather than MCMC sampling.
     Let us consider the proposed estimator shown in the Example
     \ref{ex.kl} with $\alpha=1$. Then the risk function is written as
     \begin{align*}
      &D(\tilde{p},\trad;U(z)=z\log z-z,f(z)=z)
      \\
      =&Const-
      \nmean \log \trad (\V{x}_i)
      \end{align*}
     and its gradient is written as
     \begin{equation}
      \br{\tilde{p}\psi'_{\vtheta}}-
       \br{\trad\psi'_{\vtheta}}.
     \end{equation}
     The proposed estimator replaces the mean
     $\br{\bq_{\vtheta}\psi'_{\vtheta}}$ of MLE
     with $\br{\trad\psi'_{\vtheta}}$ rather than
     $\br{\bar{p}\psi'_{\vtheta}}$.
     Note that the proposed estimator does not require the approximation
     by MCMC sampling and 
     has favorable statistical
     properties such as Fisher consistency and asymptotic efficiency.

\subsection{Relation with minimum probability flow}
The minimum probability flow (MPF) \cite{barp2019minimum} 
considers a continuous flow between probability distributions and
models the probability evolution 
from the empirical distribution to the target distribution 
using a continuous-time Markov chain, 
evaluating the distribution
along the flow at an infinitesimal time.
An estimator is constructed by minimizing the (approximated) 
KL-divergence between the empirical distribution and the evolved distribution.
Unlike the contrastive divergence, MPF does not require MCMC sampling, 
which is a major computational advantage.
Although MPF also can avoid the intractable computation of the normalization 
constant, its asymptotic efficiency is generally lower than that of the maximum likelihood estimator (MLE) or the proposed estimator \eqref{estimator}.

\section{Extension for Bayesian-like MAP Estimation}
\label{sec:bayesian_extension}

While unnormalized models are computationally efficient for point estimation, 
incorporating prior domain knowledge to the unnormalized model has remained an open challenge. 
We extend the empirical localization framework to establish a computationally tractable Bayesian-like Maximum A Posteriori (MAP) estimation framework
\footnote{A short version of this extension has been presented as a conference paper
 \cite{takenouchi2025bayesian}. We further extend the method to the deformed Bregman divergence and
add a small experiment at the end of this subsection.
}.

We introduce an unnormalized 
model $\eta(\V{x})$ ($0\leq \eta(\V{x}) <q_{\vtheta}(\V{x})$ ) on $\C{X}$
representing prior knowledge
(e.g., a constant function $\eta(\V{x}) = c$ representing 
a noninformative uniform prior over a finite domain).
By deforming the empirically localized model, 
we define the regularized localized model $\tilde{r}_{\alpha,{\vtheta},\eta}(\V{x})$ as
\begin{equation}
\tilde{r}_{\alpha,{\vtheta},\eta}(\V{x}) = \frac{\tilde{p}(\V{x})^\alpha \big( q_{\vtheta}(\V{x}) - \eta(\V{x}) \big)^{1-\alpha}}{
\br{ \tilde{p}^\alpha \big( q_{\vtheta} - \eta \big)^{1-\alpha}}}.
\label{eta-model}
\end{equation}
Note that the condition $\eta(\V{x})<q_{\vtheta}(\V{x})$ yields a restriction for 
the feasible space of the parameter and then
implementing $\max \{q_{\vtheta}(\V{x})-\eta(\V{x}),\epsilon\}$  with a small non-negative constant $\epsilon$
would be a practical approach.
The model \eqref{eta-model} can be easily calculated because the denominator of \eqref{eta-model}
is rewritten as
\begin{align}
 \br{\tilde{p}^\alpha (q_{\vtheta}-\eta)^{1-\alpha}}
=
\begin{cases}
 \sum_{\V{x}\in\{O\}}\left(\frac{n_{\V{x}}}{n}\right)^\alpha
 \left(q_{\vtheta}(\V{x})-\eta(\V{x})\right)^{1-\alpha} & \V{x}\in \C{Z}, \\
0 & \mbox{ otherwise.}
\end{cases}
\end{align}

Analogous to \eqref{estimator}, 
the proposed Bayesian-like estimator $\hat{{\vtheta}}_{U,f,\eta}$ is defined by minimizing the deformed Bregman divergence:
\begin{equation}
\hat{{\vtheta}}_{U,f,\eta} = \argmin_{{\vtheta}} \mathcal{D} \big( \tilde{r}_{\alpha,{\vtheta},\eta}, \tilde{r}_{\alpha',{\vtheta},\eta}; U, f \big).
\label{eq:bayesian_estimator}
\end{equation}

The divergence $D_{U}(\traed,\trae;U,f)$ is $0$ if 
$\trae(\V{x})=\traed(\V{x})$ holds, or equivalently
\begin{align}
 \tp(\V{x})=
 \left(
 \frac{\br{\tp^\alpha (q_{\vtheta}-\eta)^{1-\alpha}}}{\br{\tp^{\alpha'} (q_{\vtheta}-\eta)^{1-\alpha'}}}
 \right)^{\frac{1}{\alpha-\alpha'}}
 \left(q_{\vtheta}(\V{x})-\eta(\V{x})\right)
\label{equiv.condition0}
\end{align}
holds.
Terms $\br{\tp^\alpha (q_{\vtheta}(\V{x})-\eta)^{1-\alpha}}$ and 
$\br{\tp^{\alpha'} (q_{\vtheta}(\V{x})-\eta)^{1-\alpha'}}$
do not depend on $\V{x}$ and then
the condition \eqref{equiv.condition0} is rewritten as 
\begin{align}
 \tp(\V{x})=\frac{q_{\vtheta}(\V{x})-\eta(\V{x})}{\br{q_{\vtheta}-\eta}},
\label{equiv.condition0.2}
\end{align}
or equivalently
\begin{align}
 \bq_{\vtheta} (\V{x})=\left(1-\frac{\br{\eta}}{Z_{\vtheta}}\right)\tp(\V{x})+
\frac{\br{\eta}}{Z_{\vtheta}}
\frac{\eta(\V{x})}{\br{\eta}}
=
\left(1-\frac{\br{\eta}}{Z_{\vtheta}}\right)\tp(\V{x})+
\frac{\br{\eta}}{Z_{\vtheta}}
\bar{\eta}(\V{x})
\label{equiv.condition.bayes}
\end{align}
where $\bar{\eta}(\V{x})=\frac{\eta(\V{x})}{\br{\eta}}$ 
is a normalized version of $\eta(\V{x})$, satisfying $\br{\bar{\eta}}=1$.
The estimated model $\bq_{\vtheta}(\V{x})$ is represented as 
a mixture of the original distribution $\tp(\V{x})$
and the normalized $\bar{\eta}(\V{x})$.
The ratio $\frac{\br{\eta}}{Z_{\vtheta}}$ of mixture is less than $1$ because of the condition
$\eta(\V{x})<q_{\vtheta}(\V{x})$ and
we can control the ratio by multiplying a positive constant $c$ as $c\eta(\V{x})$.
Note that the normalized version $\bar{\eta}(\V{x})$ is not influenced by $c$.
The relation \eqref{equiv.condition.bayes} implies that 
$\vhtheta_{U,f,\eta}$ generally does not have the Fisher consistency 
which the original estimator \eqref{estimator} satisfies.
Some properties of such kind of mixture models have been discussed in
\cite{Copas88,TakenouchiEguchi02}
in the context of robust classification.
Also 
efficiency or robustness of estimator \eqref{eq:bayesian_estimator}
depends on a choice of functions $U$ and $f$.

A typical choice of the function $\eta(\V{x})$ is a constant not depending on $\V{x}$.
With this choice, \eqref{equiv.condition.bayes} is a mixture of the original distribution 
$\tp(\V{x})$ and the uniform distribution on $\C{X}$, and 
this may correspond to the MAP estimator with the non-informative prior distribution.
In the following, we consider 
the abstract version of the divergence 
$D(r_{\alpha,\vtheta,\eta},r_{\alpha',\vtheta,\eta};U,f)$
which is defined with the underlying distribution $p(\V{x})$ rather than the empirical distribution $\tp(\V{x})$,
and
investigate how the estimator $\vhtheta_{U,f,\eta}$ is affected by $\eta(\V{x})$.
\begin{proposition}
Assume that $p(\V{x})=\bq_{\vtheta_0}(\V{x})$ holds. Then
the estimator
$\vhtheta_{U,f,\eta}=\argmin_{\vtheta}D(r_{\alpha,\vtheta,\eta},r_{\alpha',\vtheta,\eta};U,f)$ satisfies 
\begin{align}
 \bq_{\vhtheta_{U,f,\eta}}(\V{x})=
\left(1-\frac{\br{\eta}}{{Z_{\vhtheta_{U,f,\eta}}}}\right)\bq_{\vtheta_0}(\V{x})+
\frac{\br{\eta}}{{Z_{\vtheta_{U,f,\eta}}}}
\bar{\eta}(\V{x}).
\end{align}
Also assume that 
$||\vhtheta_{U,f,\eta}-\vtheta_0||\ll 1$ holds.
Then we observe 
\begin{align}
 \vhtheta_{U,f,\eta}-\vtheta_0=V_{\vtheta_0}^{-1}\frac{\br{\eta}}{Z_{\vtheta_0}}
 \left\{\br{\bar{\eta}\psi'_{\vtheta_0}}-
\mu_{\vtheta_0}
\right\},
 \label{asymptotic}
\end{align}
where $\mu_{\vtheta_0}=\br{p\psi'_{\vtheta_0}}$ and 
 $V_{\vtheta_0}=\br{p(\psi'_{\vtheta_0}-\mu_{\vtheta_0})(\psi'_{\vtheta_0}-\mu_{\vtheta_0})^T}$ are
mean and variance of $\psi'_{\vtheta_0}(\V{x})$, respectively.
\end{proposition}

\begin{proof}
 By expanding $q_{\vhtheta_{U,f,\eta}}(\V{x})=\exp\left(\psi_{\vhtheta_{U,f,\eta}}(\V{x})\right)$ at $\vtheta_0$, we get
\begin{align}
 q_{\vhtheta_{U,f,\eta}}(\V{x})&=q_{\vtheta_0}(\V{x})+q_{\vtheta_0}(\V{x})\psi'_{\vtheta_0}(\V{x})^T
(\vhtheta_{U,f,\eta}-\vtheta_0)+{o}(||\vhtheta_{U,f,\eta}-\vtheta_0||).
\end{align}
Using the relation \eqref{equiv.condition0.2}, we have
\begin{align}
&q_{\vhtheta_{U,f,\eta}}(\V{x})\nonumber \\
=&\bq_{\vtheta_0}(\V{x})
\br{q_{\vhtheta_{U,f,\eta}}-\eta}+\eta(\V{x})\nonumber \\
=&
\bq_{\vtheta_0}(\V{x})
\br{q_{\vtheta_0}(\V{x})+q_{\vtheta_0}(\V{x})\psi'_{\vtheta_0}(\V{x})^T
(\vhtheta_{U,f,\eta}-\vtheta_0)-\eta}+\eta(\V{x})+{o}(||\vhtheta_{U,f,\eta}-\vtheta_0||).
\end{align}
By dividing these equations with ${Z_{\vtheta_0}}$ and ignoring the term ${o}(||\vhtheta_{U,\eta}-\vtheta_0||)$, 
we observe that
\begin{align}
& p(\V{x})+p(\V{x})\psi'_{\vtheta_0} (\V{x})^T(\vhtheta_{U,f,\eta}-\vtheta_0)\\
=&p(\V{x})\left(
1+\br{p\psi'_{\vtheta_0}}^T (\vhtheta_{U,f,\eta}-\vtheta_0)-
\frac{\br{\eta}}{{Z_{\vtheta_0}}}
\right)+\frac{\eta(\V{x})}{{Z_{\vtheta_0}}},
\end{align}
or equivalently
\begin{align}
 p(\V{x})\left(
\psi'_{\vtheta_0}(\V{x})
-\br{p\psi'_{\vtheta_0}}
\right)^T(\vhtheta_{U,f,\eta}-\vtheta_0)
=
\frac{\br{\eta}}{{Z_{\vtheta_0}}}
\left(
\bar{\eta}(\V{x})-p(\V{x})
\right)
\end{align}
holds.
Multiplying both sides by $\psi'_{\vtheta_0}(\V{x})
-\br{p\psi'_{\vtheta_0}}$ and summing up with respect to $\V{x}$, we conclude
\eqref{asymptotic}.
\end{proof}

The proposition implies that
degree of influence of $\eta$ is controlled by the ratio $\frac{\br{\eta}}{Z_{\vtheta_0}}$
and is modified in the direction from $\tp$ to $\bar{\eta}$.

\begin{remark}
 Another type of formulation 
\begin{align*}
 \frac{\left( \tp(\V{x}) +\eta(\V{x}) \right)^\alpha q_{\vtheta}(\V{x})^{1-\alpha} }{
\br{ (\tp+\eta)^\alpha q_{\vtheta}^{1-\alpha}}}
\end{align*}
also work as regularized estimator, 
but it requires summation over all points on domain $\C{X}$, implying 
the loss of tractability of calculation. 
\end{remark}

Figure \ref{fig2} shows behaviors of the estimator \eqref{eq:bayesian_estimator}.
We employed Bernoulli distribution $ q_{\theta}(x)=\exp(\theta x)$ 
on $\C{X}=\{+1,-1\}$,
and observed how the estimator is influenced by
the prior function $\eta(x)$.
We set $n_{+}=60, n_{-}=20,\eta(+1)=0.1,\eta(-1)=0.6$
and observed behavior of the proposed estimator by changing 
strength $c$ of the prior function $\eta$ as $c=0.2,0.5,1,1.5$.
The estimator first considers a projection from 
the empirical distribution (a red point ($\circ$)) to 
$q_{\vtheta}(x)-\eta(x)$ (a blue point on a blue dashed line).
The projected point is pulled back to Bernoulli distribution by adding $\eta(x)$
and its projection to the probability space is the proposed estimator ($+$).

Examples in Figure \ref{fig2} implies that
higher magnitude of $c$ has 
stronger effect of the prior function $\eta$ ($\triangle$) for the estimator.
Also note that a feasible set of parameter (or equivalently probability distribution)
represented by red dashed line is restricted.
 \begin{figure}[h]
\begin{center}
\includegraphics[width=.49\columnwidth,bb=0 0 779 779]{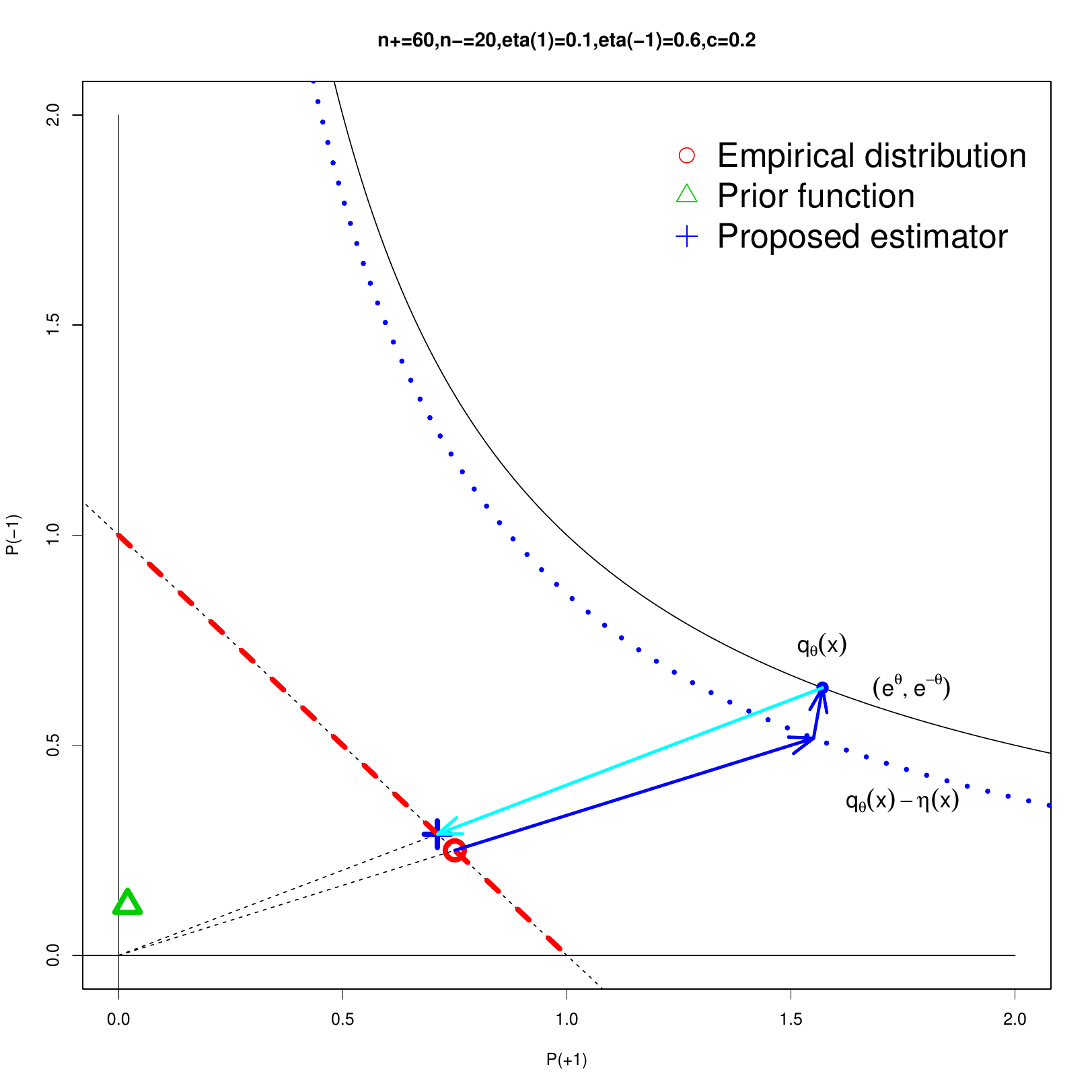}
\includegraphics[width=.49\columnwidth,bb=0 0 779 779]{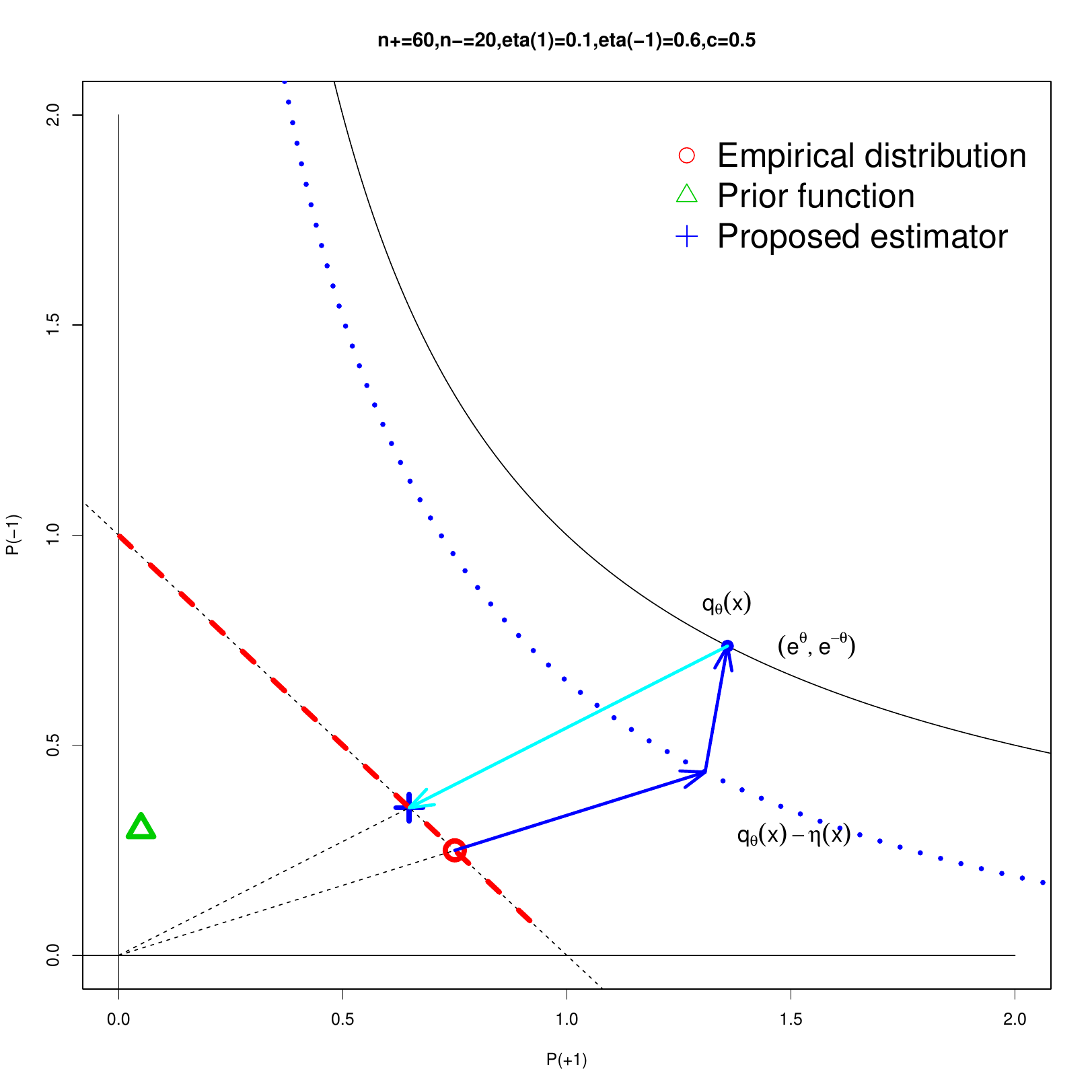}
\includegraphics[width=.49\columnwidth,bb=0 0 779 779]{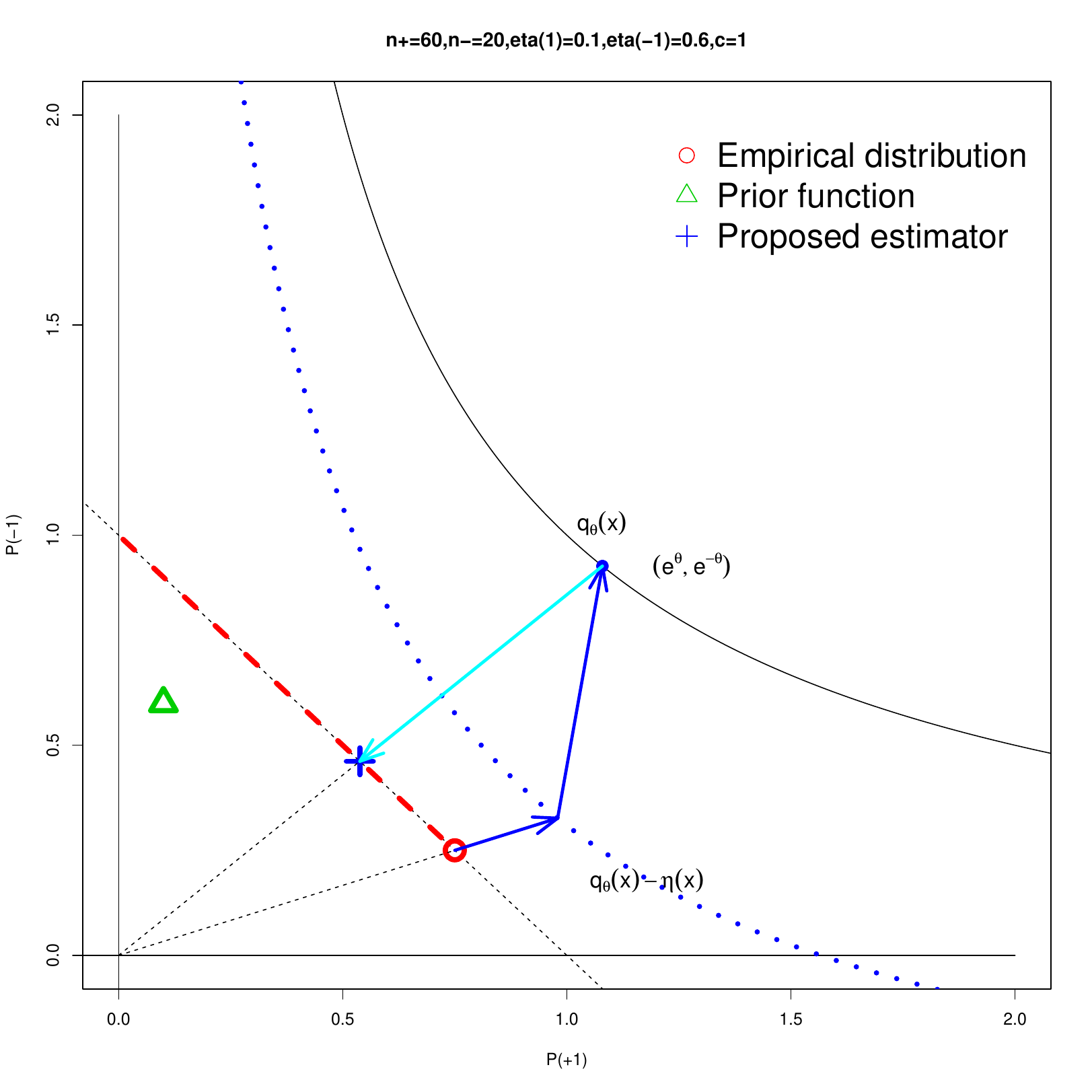}
\includegraphics[width=.49\columnwidth,bb=0 0 779 779]{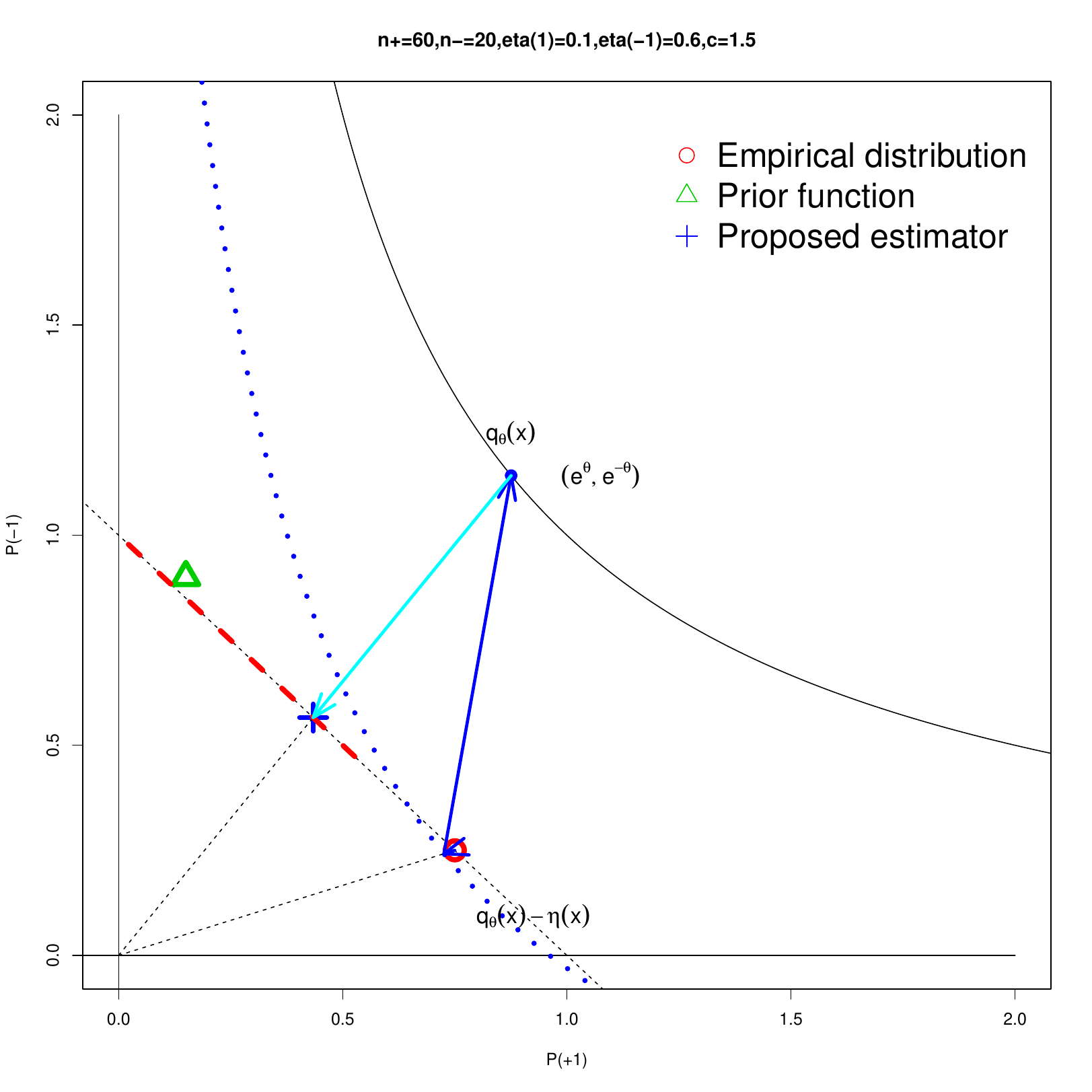}

 \caption{Examples of Bayesian-like estimation
 for  $c=0.2,0.5,1,1.5$ (upper left, upper right, lower left, lower right).
 A red point ($\circ$),  a green point ($\triangle$), and
a blue point ($+$) are
 the empirical  distribution,
 the function $\eta(x)$ representing prior 
 information, and 
 the proposed estimator, respectively.
 A black dashed line is a space of probability distribution on $\{+1,-1\}$.
 A black line represents a manifold of Bernoulli distributions
 and 
 a blue dashed line is the modified model $q_{\vtheta}(x)-\eta(x)$.
 A red dashed line is a feasible set constrained by
$\eta(x)<q_{\vtheta}(x)$.
 }
   \label{fig2}
\end{center}
 \end{figure}
    \section{Experiments}
    In this section, two small experiments have been done
    to confirm statistical efficiency (Subsection 6.1) and
    robustness against outlier (Subsection 6.2)
    of the proposed estimator.
    \subsection{Efficiency}
    We compared variance of the proposed estimator with these
    of MLE and 
    the estimator based on 
    the homogeneous H{\" o}lder divergence described in subsection \ref{takekana}
    \citep{JMLR:v18:15-596}
    using a synthetic dataset.
    The dataset was generated from Boltzmann distribution on
    $\{+1,-1\}^{10}$ whose parameter $\vtheta^{\ast}$ follows 
    $N\left(0,\frac{1}{10}I\right)$ where
    $I$ is a $10$-dimensional identity matrix.

    In the first experiments, we observed behaviour of 
    the proposed estimator against sample size $n$.
    We generated $n=50,200,800,3200,12800$ examples from the Boltzmann 
    distribution $50$ times and observed the mean square errors (MSEs).
    We employed the generating function $U(z)=\exp(z)$
    and $f(z)=\log z$ whose resultant estimator is asymptotically efficient,
    and set $\alpha=0.01$ and $\alpha'=0.99$.
    The result is shown in Figure \ref{fig.mse2n}.
    As the sample size $n$ increases, 
    the MSEs of the estimator decreases and
    the computational time increases.
This is because a size of domain $\C{Z}$ of observed examples also increases
as the sample size $n$ increases.

 \begin{figure}[ht]
  \vspace{-3mm}
\begin{center}
 \includegraphics[width=.49\columnwidth,bb=0 0 503 503]{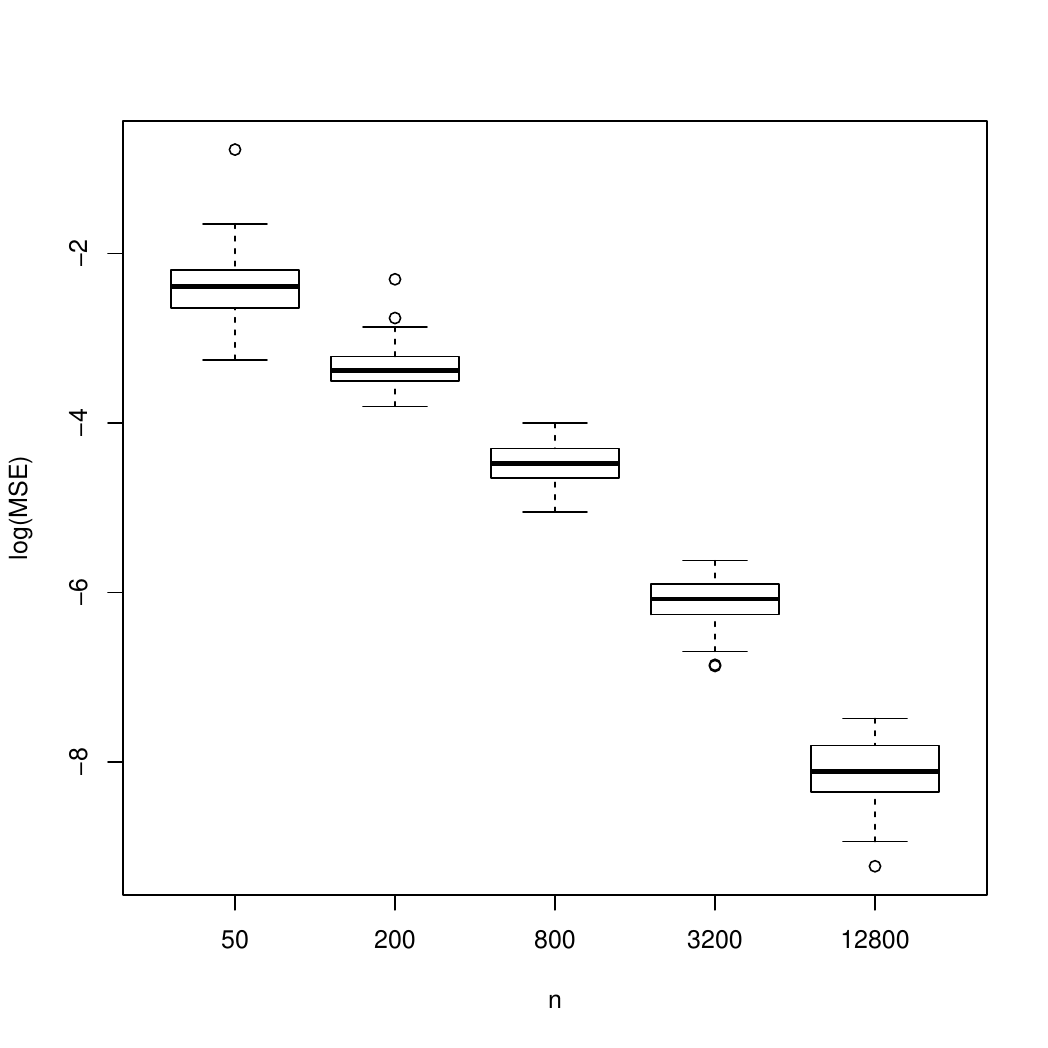}
 \includegraphics[width=.49\columnwidth,bb=0 0 503 503]{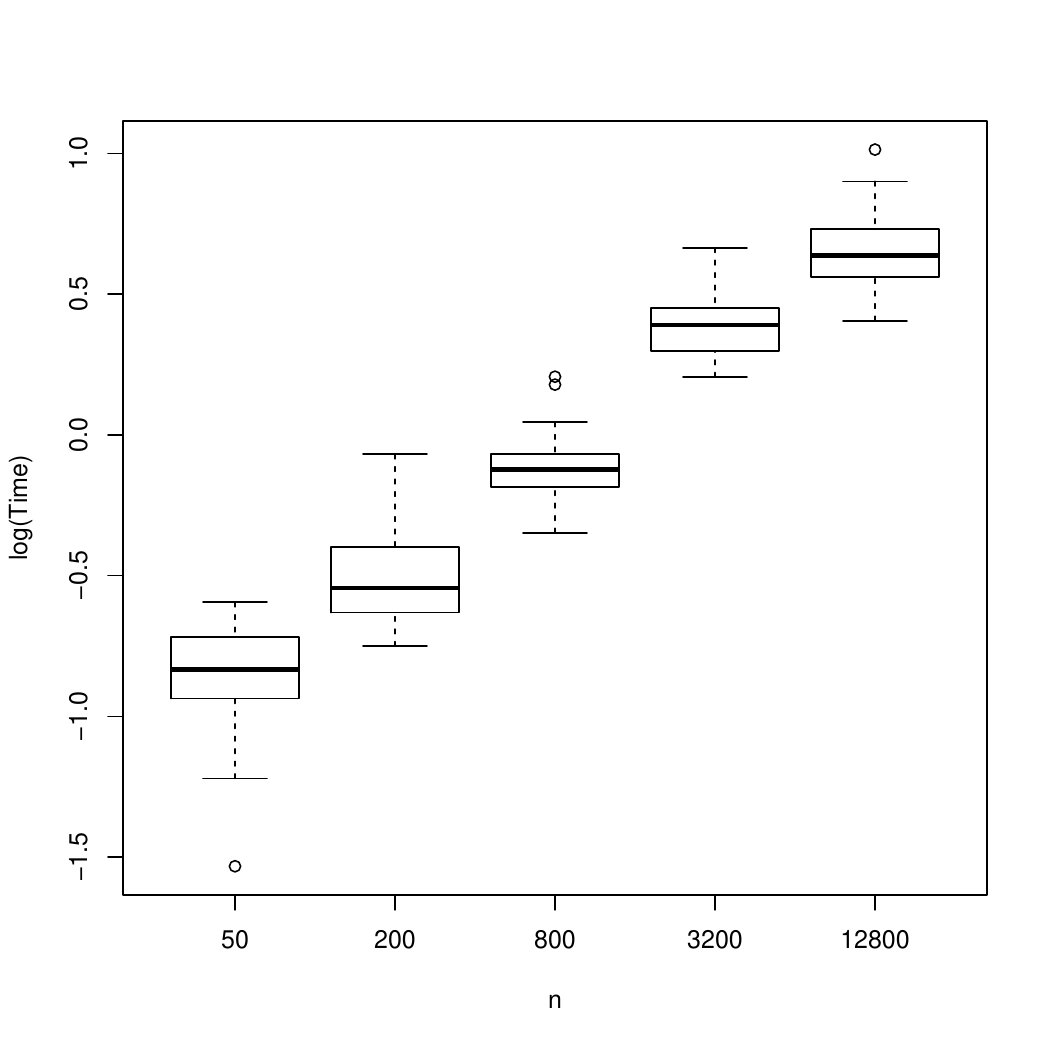}

 \caption{(a)Boxplot of MSEs of estimators in log-scale, against sample size $n$. 
 (b)Boxplot of computational times of estimators in log-scale against sample size $n$.
 }
   \label{fig.mse2n}
\end{center}
 \end{figure}

    Next, we investigated influence of choice of the generating function $U$ 
    and deformation $f$ for performance of the proposed estimator.
    We generated a dataset containing $3200$ examples 
    from the same Boltzmann distribution as above
    $50$ times,
    and compared the MSEs and computational times of 
    the proposed estimator with
    these of MLE and the estimator based on 
    the homogeneous H{\" o}lder divergence
    which is abbreviated as ``H{\" o}lder''.
    We set $\alpha=0.01$ and $\alpha'=0.99$, and 
    compared the following generating functions:
    \begin{enumerate}
     \item $U(z)=\exp(z)$: KL-divergence

     \item $U(z)=z^2/2$: Hellinger distance or equivalently 
	   $\beta$-divergence with $\beta=1$.

     \item $U(z)=-\log (-z)$: Itakura-Saito distance
    \end{enumerate}
    shown in Examples in Section \ref{proposed}.
    For the function $f$, we employed the statistical version of $f$ ($f=(U')^{-1}$)
    ~\cite{Murata_etal02} for which we can directly plug-in the empirical distribution,
    and the efficient version of $f$ presented in Theorem
    \ref{theorem.efficiency} for each function $U$. Details including
    abbreviations of combinations are shown in
    Table \ref{tbl.exp}.
    Note that for $U(z)=\exp(z)$, the statistical version of $f$
    coincides with the efficient
    version and for $U(z)=-\log (-z)$, we employed $f_+(z)=-\exp(2\sqrt{z})$
    because $D(p,q;U,f_+)=D(q,p;U,f_-)$ holds (see Example \ref{example.is}).
     \begin{table}[h]
\caption{Combination of $U$ and $f$, and corresponded abbreviations.}
      \label{tbl.exp}
\begin{center}
\begin{small}
\begin{sc}
       \begin{tabular}{|c|c|c|}
     \hline
      $U$ & Stat. ver. of $f$ & Efficient ver. of $f$  \\
      \hline
      $\exp(z)$ & KL: $\log z$ & KL: $\log z$ \\
      \hline
     $z^2/2$ & Hel: $z$ & F-Hel: $2\sqrt{z}$ \\       \hline
      $-\log (-z)$ & IS: $-1/z$ & F-IS: $-\exp(2\sqrt{z})$
		\\ \hline
       \end{tabular}
 \end{sc}
\end{small}
\end{center}
\end{table}

     Figure \ref{fig1}(a) shows a boxplot of MSEs between
     $\vtheta^{\ast}$ and each estimator
     over $50$ trials, in log-scale.
 \begin{figure}[ht]
  \vspace{-3mm}
\begin{center}
 \includegraphics[width=.49\columnwidth,bb=0 0 503 503]{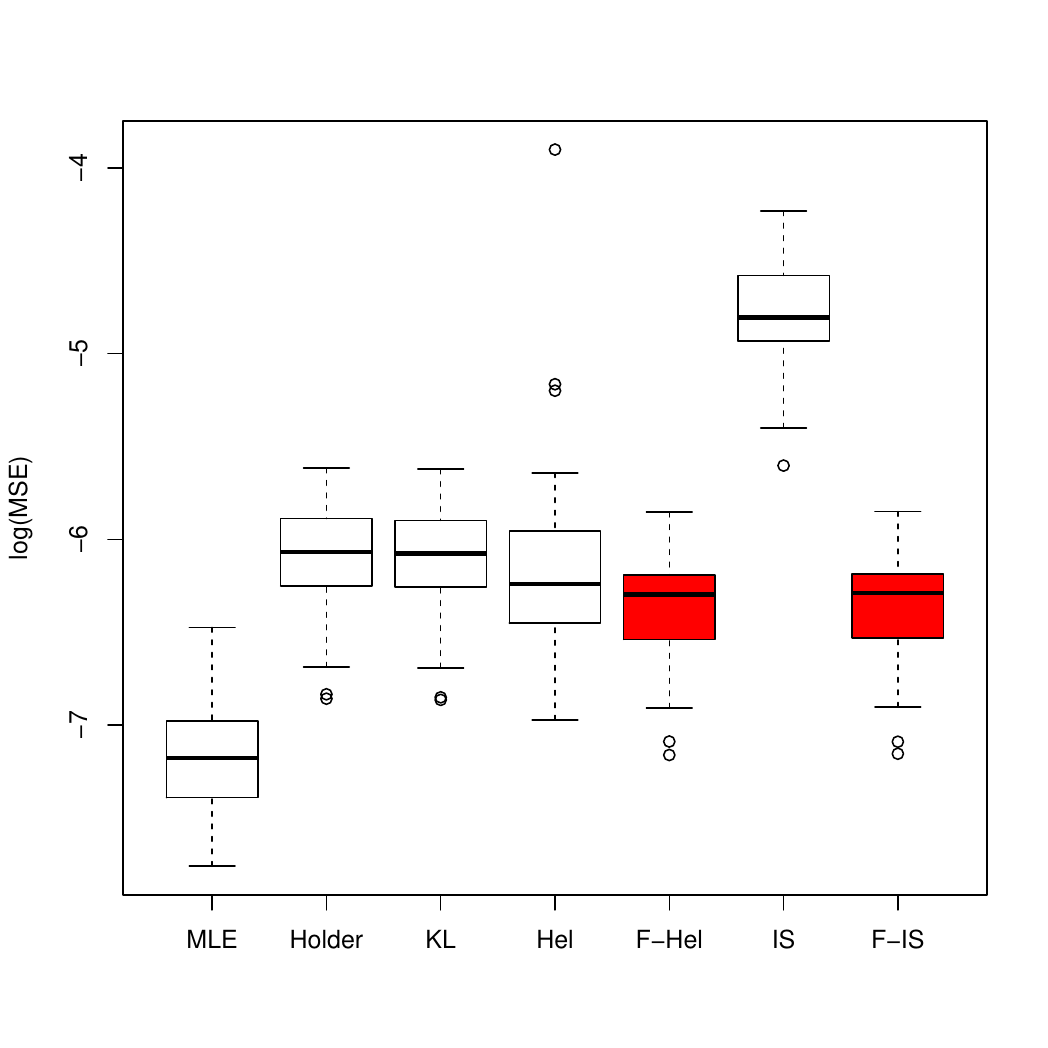}
 \includegraphics[width=.49\columnwidth,bb=0 0 503 503]{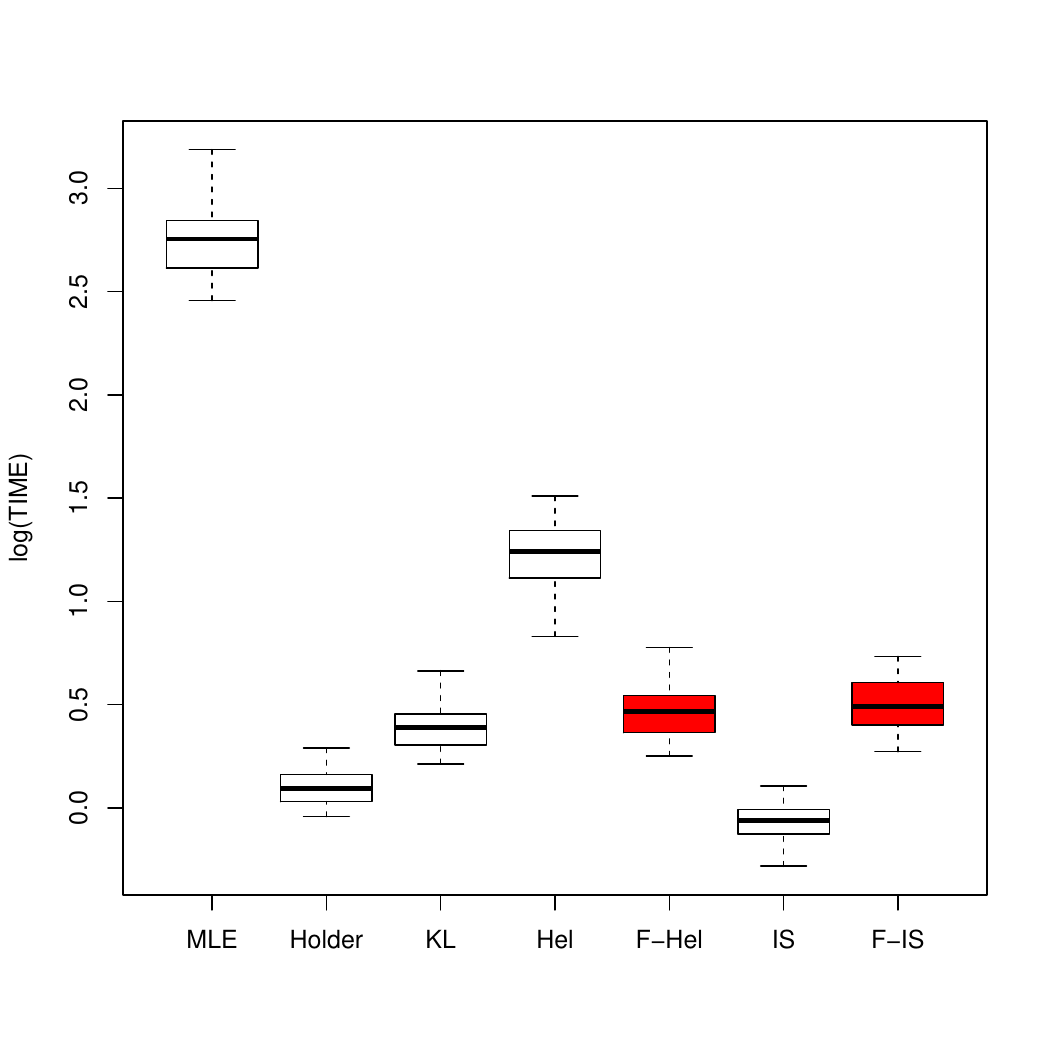}
 \caption{(a)Boxplot of MSEs of estimators, in log-scale. 
 Efficient version of $f$s are colored red.
 (b)Boxplot of computational times of estimators, in log-scale.
 Efficient version of $f$s are colored red.
 }
   \label{fig1}
\end{center}
 \end{figure}
	 Regardless of choice of the function $U$,
	 efficient versions of $f$ improve performance compared with 
	 the statistical version of $f$ and
	 attain comparable estimation error
     	 with MLE.
	 Figure \ref{fig1}(b) indicates a boxplot of computational
     	 times of estimators over $50$ trials, in log-scale.
	 We observe that 
	 computational time of the proposed estimator
     	 is drastically reduced compared with that of MLE while
	 the efficient version of proposed estimators maintain the
	 same level of estimation error with MLE.
	 This is because
	 the proposed estimator does not require the computation of the
     	 normalization constant.
	 Applying one-step estimator \cite{vandervaart1998as}
	 to the proposed estimator might further improve performance.
	 
	  \subsection{Robustness}
	  We investigated robustness of the proposed estimator shown in Section
	  \ref{sec.robust} with small synthetic datasets.
	  In the first experiment,
	  we generated datasets containing $n=25,50,100,200,400,800$
	  examples $50$ times,
	  from geometric distribution $\bq_{\theta_0}(x)=(1-\theta_0)
	  \theta_0^{x-1}$ on $\C{X}=\{1,2,\ldots\}$ and compared the
	  robust version of proposed estimator shown in Example
	  \ref{example.robust1} with 
	  $\alpha=0.01,\alpha'=0.99$ to MLE.
	  Figure \ref{fig.robust1} shows 
	  a boxplot of MSEs of MLE and the proposed estimator
	  against sample size $n$, in log-scale.
	  Note that in this experiment, 
	  MLE outperforms the proposed estimator
	  because the proposed estimator is not
	  asymptotically efficient and 
	  datasets is not contaminated by outlier noise.	   
	  In the second experiment, we added an outlier noise 
 	  ($2\times$ (average of examples without outlier)) 
	  to an example
	  and observed behaviors of estimators, whose result
	  is shown in Figure \ref{fig.robust2}.
	  As predicted by the theory, the robust version of proposed
	  estimator is not so influenced while MLE gets worse because of the outlier noise.

	  	 \begin{figure}[ht]
		  \vspace{-3mm}
	  \begin{center}
\includegraphics[width=.49\columnwidth,bb=0 0 503 503]{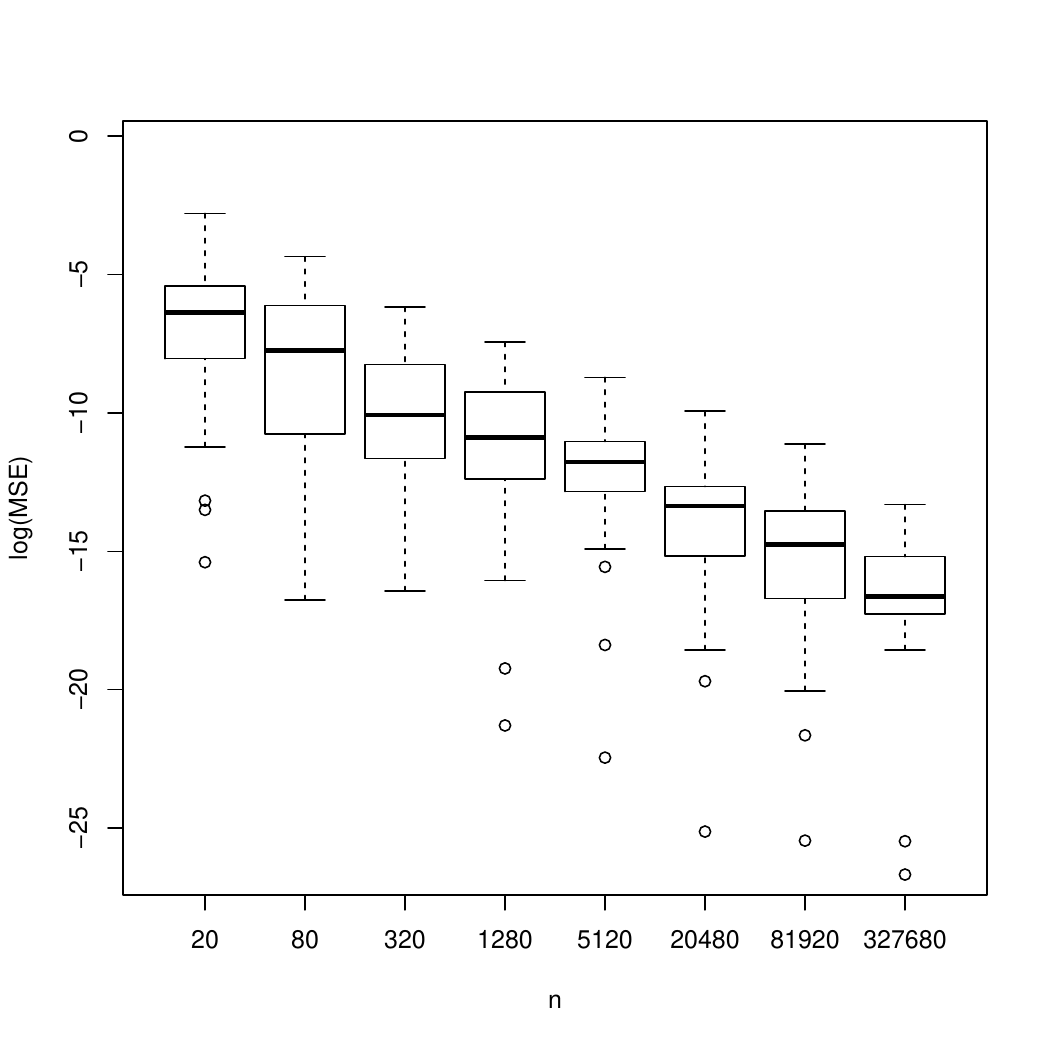}
\includegraphics[width=.49\columnwidth,bb=0 0 503 503]{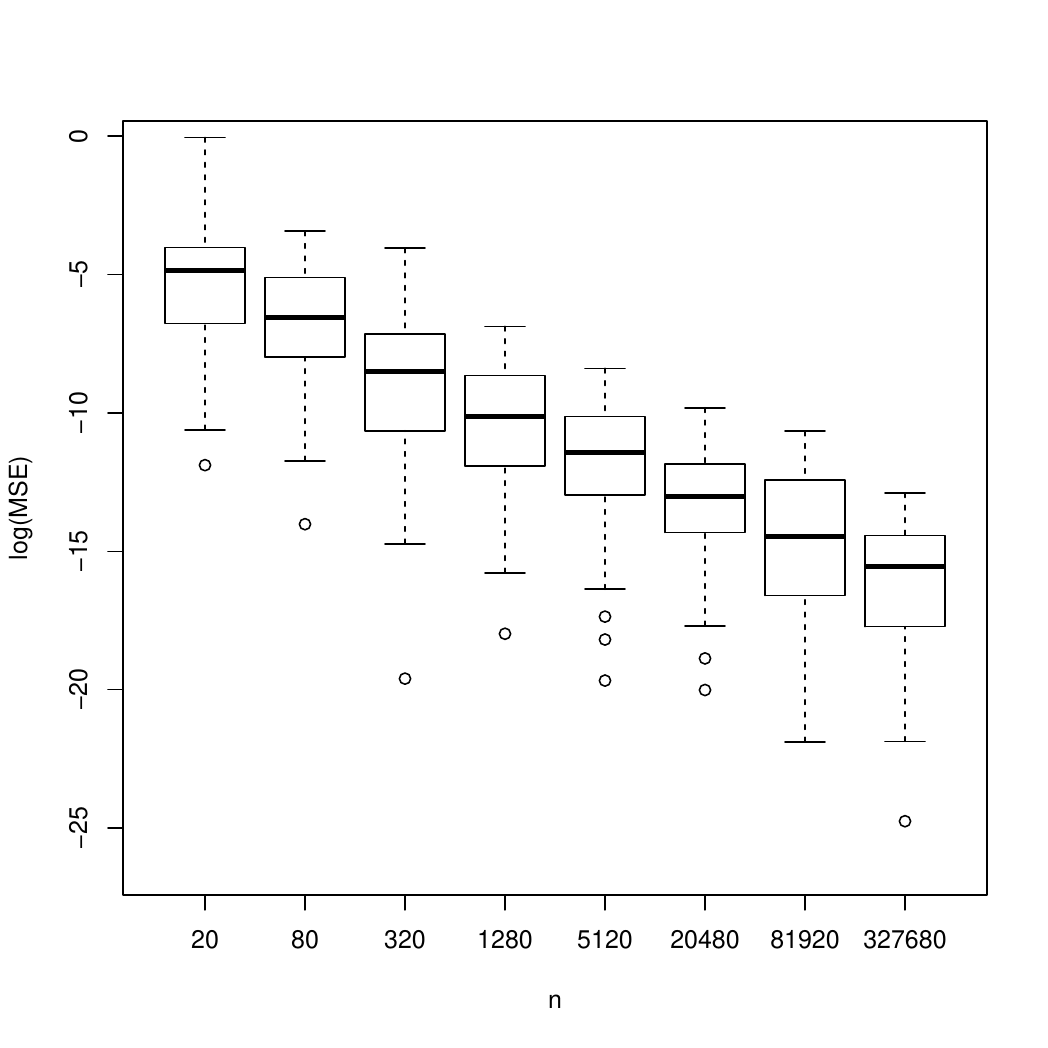}

	   \caption{
	   (a)Boxplot of MSEs of MLE against $n$ in log-scale, for
	   datasets without outlier noise.
	   (b)Boxplot of MSEs of proposed estimator against $n$ in log-scale, for
	   datasets without outlier noise.
	   }
       \label{fig.robust1}
	  \end{center}
	 \end{figure}
	  	 \begin{figure}[ht]
		  \vspace{-3mm}
	  \begin{center}

\includegraphics[width=.49\columnwidth,bb=0 0 503 503]{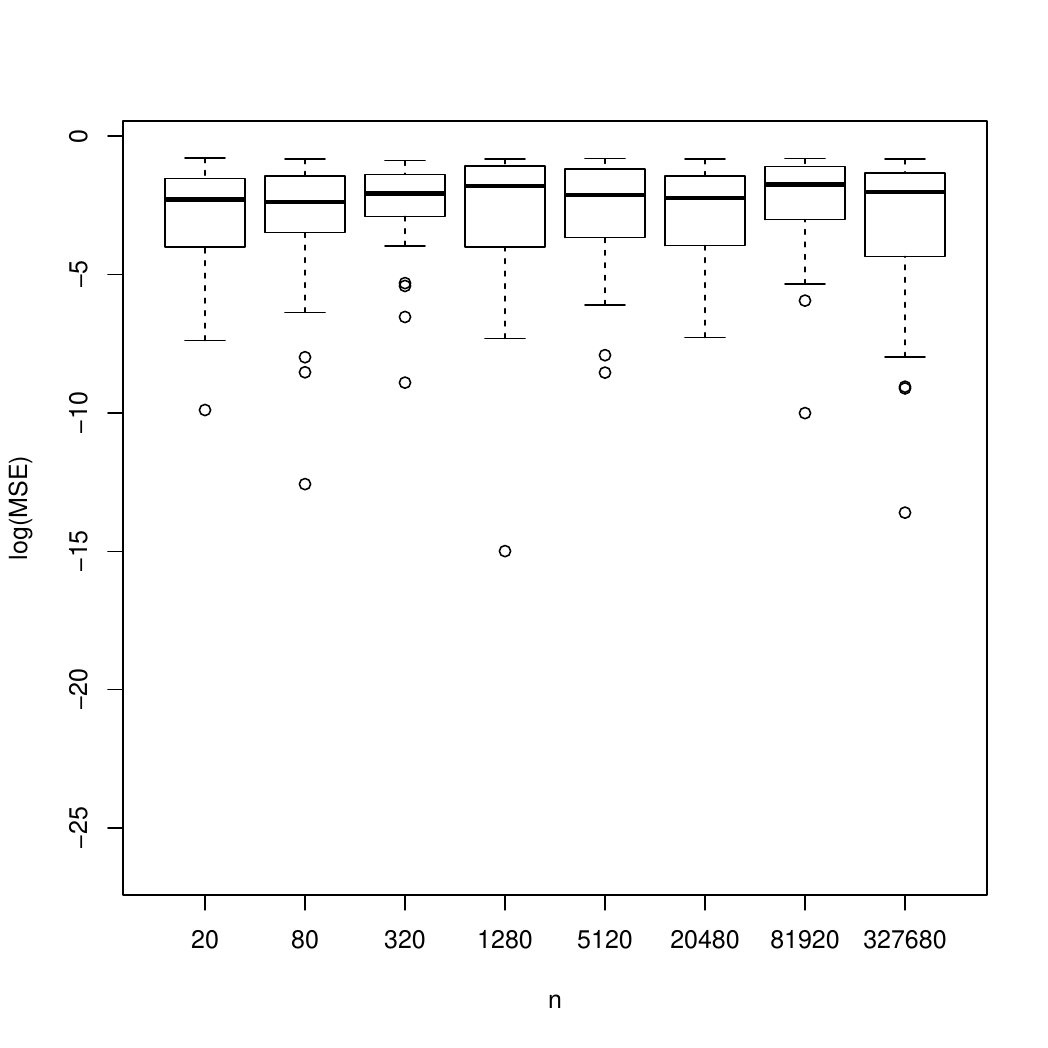}
\includegraphics[width=.49\columnwidth,bb=0 0 503 503]{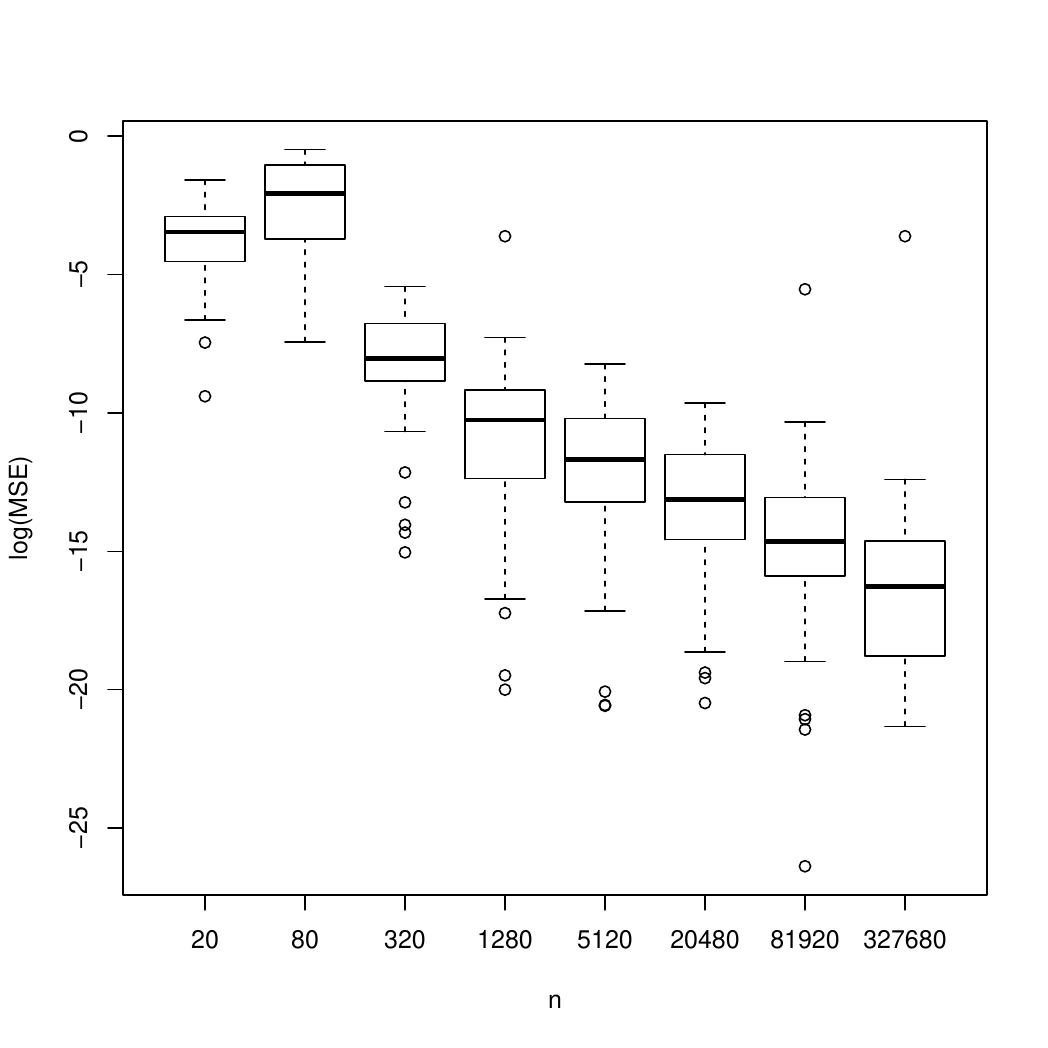}

	   \caption{
	   (a)Boxplot of MSEs of MLE against $n$ in log-scale, for
	   datasets with outlier noise.
	   (b)Boxplot of MSEs of proposed estimator against $n$ in log-scale, for
	   datasets with outlier noise.
	   }
       \label{fig.robust2}
	  \end{center}
	 \end{figure}

    	  \section{Conclusions}
	  We proposed the estimator for probabilistic
    	  models on discrete space.
	  The proposed estimator is derived from
	  the combination of the deformed Bregman
	  divergence and the empirically localization of model.
 	  The empirically localized model makes it possible to construct the computationally
    	  feasible estimator which does not require the calculation of
    	  the normalization constant.
	  In addition, we show that the specific form of deformation 
	  ensures asymptotic efficiency or robustness against outlier noise for
 	  the proposed estimator and
	  numerically confirm performance of the proposed estimator.

\section*{Acknowledgments}
This work was supported by JSPS KAKENHI Grant Number 
26K14740.

\onecolumn

 \appendix

\section*{Appendix A: Proof of Lemma \ref{lemma.1}}\label{app.a}
    Let us assume that the empirical distribution is written as
    $\tilde{p}(\V{x})=\bq_{\vtheta_0}(\V{x})+\epsilon s(\V{x})$
    where $\epsilon$ is a small constant and $s(\V{x})$ is a function
     satisfying $\br{s}=0$.
     By expanding
     \eqref{equiv.condition}
     around $\vtheta=\vtheta_0$,
    we obtain   
     \begin{align}
      0\simeq &
      \left.
      \frac{\partial }{\partial \vtheta}
      D(\tra,\trad;U,f)
      \right|_{\vtheta=\vtheta_0}+
      \left.\frac{\partial^2}{\partial \vtheta\partial \vtheta^T}
      D(\tra,\trad;U,f)\right|_{\vtheta=\vtheta_0}
      (\vhtheta-\vtheta_0).
     \label{app1.eq}
     \end{align}         
     By the delta method
     \cite{vandervaart1998as},
     we observe that
     the first term can be further expanded as follows.
     \begin{align*}
      &\left.
      \frac{\partial }{\partial \vtheta}
      D(\tra,\trad;U,f)
      \right|_{\vtheta=\vtheta_0}\\
      \\ 
      =& \left<
      U'(f(\trad))f'(\trad)\frac{\partial \trad}{\partial \vtheta}
      -U'(f(\tra))f'(\tra)\frac{\partial \tra}{\partial \vtheta}     
      \right. \\
      & 
       -U''(f(\tra))f'(\tra)\frac{\partial\tra}{\partial \vtheta}(f(\trad)-f(\tra))
      \\
      &\left.\left.
      -U'(f(\tra))\left(f'(\trad)\frac{\partial\trad}{\partial \vtheta}
      -f'(\tra)\frac{\partial \tra}{\partial \vtheta}\right)
      \right>
      \right|_{\vtheta=\vtheta_0}
      \\
      =&
      \left.
      \frac{\partial }{\partial \vtheta}
      D(\tra,\trad;U,f)
      \right|_{\vtheta=\vtheta_0,\epsilon=0}      
      +
      \epsilon \left.
      \frac{\partial }{\partial \epsilon}\frac{\partial }{\partial \vtheta}
      D(\tra,\trad;U,f)
      \right|_{\vtheta=\vtheta_0,\epsilon=0}      
      +\mathcal{O}(\epsilon^2)
      \\
      =&
      0
      +
      \epsilon \left.
      \frac{\partial }{\partial \epsilon}\frac{\partial }{\partial \vtheta}
      D(\tra,\trad;U,f)
      \right|_{\vtheta=\vtheta_0,\epsilon=0}      
      +\mathcal{O}(\epsilon^2)
      \end{align*}
      where
      \begin{align*}
       &\frac{\partial }{\partial \epsilon}\frac{\partial }{\partial \vtheta}
       D(\tra,\trad;U,f)
       \\
       =&
      \left<
      \left\{
      U''(f(\trad))f'(\trad)^2   
      \ptrad\petrad^T
       +U'(f(\trad))f''(\trad)     
      \ptrad\petrad^T      
      \right.
       \right.
      \\
      &
       +U'(f(\trad))f'(\trad)   \frac{\partial}{\partial \epsilon} \ptrad
       -
       U''(f(\tra))f'(\tra)^2   
       \ptra\petra^T
       \\
       &
       -U'(f(\tra))f''(\tra)     
       \ptra\petra^T
       -U'(f(\bqz))f'(\bqz)   \frac{\partial}{\partial \epsilon} \ptra
      \\
       &-
       \frac{\partial U''(f(\tra))f'(\tra)\ptra}{\partial \epsilon} (f(\trad)-f(\tra))
       \\
      &-
       U''(f(\tra))f'(\tra)\ptra \left(
      f'(\trad)\petrad^T-f'(\tra)\petra^T
      \right)      
      \\
      &
       -
      U''(f(\tra))f'(\tra)\left(
       f'(\trad)\ptrad-f'(\tra)\ptra
      \right)\petra^T
      \\
      &-
      U'(f(\bqz))\left(
       f''(\trad)\ptrad\petrad^T-f''(\tra)\ptra\petra^T
      \right)      
      \\
      &-
       \left. \left. 
       U'(f(\tra))
      \left(
       f'(\trad)\frac{\partial}{\partial \epsilon} \ptrad-
       f'(\tra)\frac{\partial}{\partial \epsilon} \ptra
      \right)
       \right\}  s 
       \right> 
      \end{align*}
      Note that $r_{\alpha,\vtheta_0}(\V{x})=\bar{q}_{\vtheta_0}(\V{x})$ regardless of value of $\alpha$ 
     and 
     \begin{align*}
\left.\frac{\partial \tra}{\partial \vtheta}\right|_{\vtheta=\vtheta_0,\epsilon=0}&=
      (1-\alpha)\bqz(\V{x})\left\{ \psi_{\vtheta_0}-\vmu_0\right\},
     \\
      \left.
      \frac{\partial \tilde{r}_{\alpha,\vtheta_0}}{\partial \epsilon}
      \right|_{\epsilon=0}
      &=\alpha
      \left\{
      \frac{\bar{q}_{\vtheta_0}(\V{x})^{\alpha-1}q_{\vtheta_0}(\V{x})^{1-\alpha}}{\br{\bar{q}_{\vtheta_0}(\V{x})^{\alpha}q_{\vtheta_0}(\V{x})^{1-\alpha}}}
      s(\V{x})
      -
      \bq_{\vtheta_0}(\V{x})
      \frac{\br{\bar{q}_{\vtheta_0}^{\alpha-1}q_{\vtheta_0}^{1-\alpha}s}}{\br{\bar{q}_{\vtheta_0}^{\alpha}q_{\vtheta_0}^{1-\alpha}}}
      \right\}
      \\
      &=
      \alpha
      \left\{
      s(\V{x})-\bq_{\vtheta_0}(\V{x})\br{s}
      \right\}
      \\
      &=
      \alpha s(\V{x})
      \end{align*}
      holds.  In the last line, we use $\br{s}=0$.
Then we observe that
      \begin{align}
       &
       \left.
       \frac{\partial }{\partial \epsilon}\frac{\partial }{\partial \vtheta}
       D(\tra,\trad;U,f)
       \right|_{\vtheta=\vtheta_0,\epsilon=0}
       \nonumber\\       
       =&
      \left<
       \left\{
      U''(f(\bqz))f'(\bqz)^2   
       \left.
       \ptrad\petrad^T
       \right|_{\vtheta=\vtheta_0,\epsilon=0}
      +U'(f(\bqz))f''(\bqz)     
       \left.
       \ptrad\petrad^T
       \right|_{\vtheta=\vtheta_0,\epsilon=0}
       \right.
       \right.
              \nonumber\\
      &
       +U'(f(\bqz))f'(\bqz)   
       \left.
       \frac{\partial}{\partial \epsilon} \ptrad      
       \right|_{\vtheta=\vtheta_0,\epsilon=0}
       -
       U''(f(\bqz))f'(\bqz)^2   
       \left.
      \ptra\petra^T
       \right|_{\vtheta=\vtheta_0,\epsilon=0}
              \nonumber\\ 
       &
       -U'(f(\bqz))f''(\bqz)     
       \left.
       \ptra\petra^T
       \right|_{\vtheta=\vtheta_0,\epsilon=0}
      -U'(f(\bqz))f'(\bqz)   
       \left.
       \frac{\partial}{\partial \epsilon} \ptra
       \right|_{\vtheta=\vtheta_0,\epsilon=0}
              \nonumber\\
      &-
       U''(f(\bqz))f'(\bqz)
       \left.
       \ptra
       \left(
       f'(\bqz)
       \petrad^T
       -f'(\bqz)
       \petra^T
      \right)
       \right|_{\vtheta=\vtheta_0,\epsilon=0}
              \nonumber\\
      &
       -
      U''(f(\bqz))f'(\bqz)^2
       \left.
       \left(
      \ptrad-\ptra
      \right)\petra^T
       \right|_{\vtheta=\vtheta_0,\epsilon=0}
              \nonumber\\
      &-
      U'(f(\bqz))f''(\bqz)
       \left.
       \left(
      \ptrad\petrad^T-\ptra\petra^T
      \right)      
       \right|_{\vtheta=\vtheta_0,\epsilon=0}
              \nonumber\\
      &-\left.
       \left.
      U'(f(\bqz))f'(\bqz)
       \left.
      \left(
      \frac{\partial}{\partial \epsilon} \ptrad-
      \frac{\partial}{\partial \epsilon} \ptra
      \right)
       \right|_{\vtheta=\vtheta_0,\epsilon=0}
       \right\}
       s
       \right>      
              \nonumber\\
       =&
       \br{
       U''(f(\bqz))f'(\bqz)^2
       \left.
       \left(
       \ptrad-\ptra
       \right)
       \left(       
       \petrad-\petra
       \right)
       \right|_{\vtheta=\vtheta_0,\epsilon=0}
       s
       }
              \nonumber\\
      =&- (\alpha-\alpha')^2
      \br{\xi_{U,f}(\bq_{\vtheta_0})(\psi_{\vtheta_0}'-\vmu_0)
      (\tilde{p}-\bq_{\vtheta_0})
      }
       \label{eq.d2}
       .
      \end{align}
     From the central limit theorem, 
     \begin{align*}      
      \sqrt{n}\br{\xi_{U,f}(\bq_{\vtheta_0})(\psi_{\vtheta_0}'-\vmu_0)
      (\tilde{p}-\bq_{\vtheta_0})
      }
      =&
      \frac{1}{\sqrt{n}}\nsum\zeta_{U,f,\vtheta_0}(\V{x}_i)
      \end{align*}
     asymptotically follows the normal  distribution with mean $\V{0}$
     and variance
     $J_{U,f,\vtheta_0}=\br{\bq_{\vtheta_0}\zeta_{U,f,\vtheta_0}\zeta_{U,f,\vtheta_0}^T}$.

      For the second term of \eqref{app1.eq}, we observe that
      \begin{align*}
       &
       \left.\frac{\partial^2}{\partial \vtheta\partial \vtheta^T}
      D(\tra,\trad;U,f)\right|_{\vtheta=\vtheta_0,\epsilon=0}
       \\ 
       =& \frac{\partial}{\partial \vtheta}\left<
      U'(f(\trad))f'(\trad)\frac{\partial \trad}{\partial \vtheta}^T
      -U'(f(\tra))f'(\tra)\frac{\partial \tra}{\partial \vtheta}^T
      \right. \\
      & 
       -U''(f(\tra))f'(\tra)\frac{\partial\tra}{\partial \vtheta}^T(f(\trad)-f(\tra))
       \\
       &
       \left.\left.
       -U'(f(\tra))\left(f'(\trad)\frac{\partial\trad}{\partial \vtheta}^T
      -f'(\tra)\frac{\partial \tra}{\partial \vtheta}^T\right)
      \right>
       \right|_{\vtheta=\vtheta_0,\epsilon=0}
       \\
       =& 
       \left<
       U''(f(\bqz))f'(\bqz)^2
       \left.
       \ptrad\ptrad^T
       \right|_{\vtheta=\vtheta_0,\epsilon=0}
       +U'(f(\bqz))f''(\bqz)
       \left.
       \ptrad\ptrad^T
       \right|_{\vtheta=\vtheta_0,\epsilon=0}
       \right. \\
       &
       +U'(f(\bqz))f'(\bqz)
       \left.
       \frac{\partial^2 \trad}{\partial\vtheta\partial \vtheta^T}
       \right|_{\vtheta=\vtheta_0,\epsilon=0}
       -U''(f(\bqz))f'(\bqz)^2
       \left.\ptra\ptra^T\right|_{\vtheta=\vtheta_0,\epsilon=0}
       \\
       &-U'(f(\bqz))f''(\bqz)
       \left.\ptra\ptra^T\right|_{\vtheta=\vtheta_0,\epsilon=0}
       -U'(f(\bqz))f'(\bqz)
       \left.\frac{\partial^2 \tra}{\partial\vtheta\partial \vtheta^T}\right|_{\vtheta=\vtheta_0,\epsilon=0}
       \\
       &
       -\left.\frac{\partial}{\partial\vtheta}
       \left(
       U''(f(\tra))f'(\tra)\frac{\partial\tra}{\partial \vtheta}^T
       \right)\right|_{\vtheta=\vtheta_0,\epsilon=0} (f(\bqz)-f(\bqz))
       \\
       &
       -
       U''(f(\bqz))f'(\bqz)
       \left.\left(f'(\bqz)\ptrad -f'(\bqz)\ptra\right)
       \frac{\partial\tra}{\partial \vtheta}^T\right|_{\vtheta=\vtheta_0,\epsilon=0}
       \\
      & 
       -U''(f(\bqz))f'(\bqz)^2
       \left.\ptra\left(\ptrad^T-\ptra^T
       \right)\right|_{\vtheta=\vtheta_0,\epsilon=0}
       \\
       &
       -U'(f(\bqz))f''(\bqz)
       \left.\left(\ptrad\ptrad^T-\ptra\ptra^T
       \right)\right|_{\vtheta=\vtheta_0,\epsilon=0}
       \\
       &
       \left.
       -U'(f(\bqz))f'(\bqz)
       \left.
       \left(
       \frac{\partial^2 \trad}{\partial\vtheta\partial \vtheta^T}
       -
       \frac{\partial^2 \tra}{\partial\vtheta\partial \vtheta^T}
       \right)
       \right|_{\vtheta=\vtheta_0,\epsilon=0}      
      \right>
       \\
       =&
       \br{
       \left.
       U''(f(\bqz))f'(\bqz)^2
       \left(\ptrad-\ptra\right)
       \left(\ptrad-\ptra\right)^T
       \right|_{\vtheta=\vtheta_0,\epsilon=0}      
       }
       \\
       =&
       (\alpha'-\alpha)^2\br{U''(f(\bqz))f'(\bqz)^2\bqz^2(\psi'_{\vtheta_0}-\vmu_0)
       (\psi'_{\vtheta_0}-\vmu_0)^T}
       \\
       =&
       (\alpha'-\alpha)^2H_{U,f,\vtheta_0}.
      \end{align*}
      From the law of large numbers, $\tilde{p}$ converges to $\bqz$, {\it i.e.},
      $\epsilon\to 0$ when $n\to \infty$, then we have
     \begin{align*}
      \left.\frac{\partial^2}{\partial \vtheta\partial \vtheta^T}
      D(\tra,\trad;U,f)\right|_{\vtheta=\vtheta_0}
       \to
      (\alpha'-\alpha)^2H_{U,f,\vtheta_0}      
     \end{align*}
     in the limit of $n\to \infty$.
     Taking the probabilistic error in the law of large numbers into
     account,
     we obtain an equality     
   \begin{equation*}
    \sqrt{n}(\vhtheta_{U,f}-\vtheta_0)
     =H_{U,f,\vtheta_0}      ^{-1}
     \frac{1}{\sqrt{n}}\nsum
     \zeta_{U,f,\theta_0}(\V{x}_i)
     +o_p(1).
   \end{equation*}
     Consequently, the asymptotic distribution of the estimator is given
     as
     $\sqrt{n}(\vhtheta_{U,f}-\vtheta_0)\sim
     N(0,V_{U,f})$.

	  \bibliographystyle{unsrt}

\end{document}